\documentclass{article} 
\usepackage{iclr2027_conference,times}

\usepackage{amsmath,amsfonts,bm}

\def\eqref#1{\textup{(\ref{#1})}}

\def\1{\bm{1}}

\DeclareMathAlphabet{\mathsfit}{\encodingdefault}{\sfdefault}{m}{sl}
\SetMathAlphabet{\mathsfit}{bold}{\encodingdefault}{\sfdefault}{bx}{n}

\usepackage{amssymb} 
\usepackage{amsthm} 
\usepackage{booktabs}
\usepackage{multirow}
\usepackage{tabularx}
\usepackage{wrapfig}
\usepackage{algorithm}
\usepackage{algpseudocode}
\usepackage[table]{xcolor}
\usepackage{tcolorbox}
\newtheorem{theorem}{Theorem}
\newtheorem{lemma}{Lemma}
\newtheorem{example}{Example}
\usepackage{hyperref}
\usepackage{url}
\usepackage{titletoc}
\usepackage{graphicx}
\definecolor{benchmarkgroup}{RGB}{245,245,245}
\definecolor{adastephighlight}{RGB}{234,243,253}
\definecolor{promptplaceholder}{RGB}{150,66,5}
\definecolor{promptthought}{RGB}{0,105,48}
\definecolor{promptaction}{RGB}{112,44,166}
\newcommand{\promptfield}[1]{\textcolor{promptplaceholder}{\texttt{\{#1\}}}}

\title{AdaStep: Adaptive Step Credit Weighting for \\ Agentic Reinforcement Learning}

\author{
Xin Wang$^{1,3}$\thanks{Equal contribution.} \enspace Wenhao Wu$^{2,3}$\footnotemark[1] \enspace Menghao Zhang$^{3}$\footnotemark[1] \enspace Zhi Wang$^{2}$\thanks{Corresponding authors.} \enspace Kun Shao$^{3}$\footnotemark[2] \enspace Jian Luan$^{3}$\footnotemark[2] \\[0.5em]
\textnormal{$^{1}$Tsinghua University \quad $^{2}$Nanjing University} \quad \textnormal{$^{3}$Xiaomi Inc.} \\
\textnormal{$^{\dagger}$\texttt{zhiwang@nju.edu.cn}, \texttt{\{shaokun,luanjian\}@xiaomi.com}}
}

\hypersetup{
    hidelinks,
    pdftitle={AdaStep: Adaptive Step Credit Weighting for Agentic Reinforcement Learning},
    pdfauthor={Xin Wang, Wenhao Wu, Menghao Zhang, Zhi Wang, Kun Shao, Jian Luan}
}

\iclrfinalcopy
\begin{document}

\maketitle
\lhead{Preprint}

\begin{abstract}
Long-horizon LLM agents are typically trained with sparse outcome rewards, making trajectory-level objectives too coarse to distinguish the contribution of individual decisions.
Step-level credit assignment provides finer-grained supervision, but its estimates can be unreliable because observed returns also depend on subsequent actions, environment transitions, and trajectory length.
We propose \textbf{AdaStep}, an \textbf{Ada}ptive \textbf{Step}-credit weighting method that controls how strongly each group-derived local advantage modifies the trajectory-level signal.
We formulate this weighting as a mean-squared-error estimation problem for the latent step advantage and, under an explicit conditional sampling assumption, derive an optimal per-state shrinkage coefficient.
The coefficient admits a \textbf{signal-to-total-variance} interpretation: it preserves local credit when return variation is attributable to the selected action and suppresses it when variation is dominated by downstream randomness. 
AdaStep requires only lightweight scalar computation, with no critic, additional rollouts, or extra model inference.
Experiments with three model backbones on ALFWorld, WebShop, and ScienceWorld show \textbf{consistent improvements} over baselines at low computational cost.
\end{abstract}

\section{Introduction}

\begin{wrapfigure}[22]{r}{0.50\textwidth}
\vspace{-\intextsep}
\centering
\includegraphics[width=\linewidth]{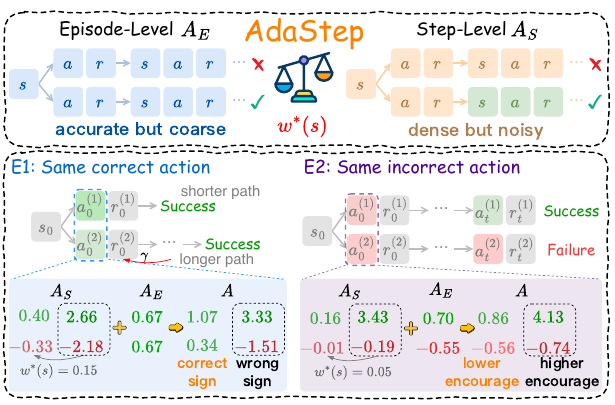}
\setlength{\abovecaptionskip}{4pt}
\caption{\small The episode-level advantage provides an accurate but coarse signal of trajectory outcomes, while the step-level advantage offers finer-grained but noisy local credit. In GiGPO, combining these signals with equal weights can assign negative advantages to useful actions or excessively large positive advantages to actions unrelated to success. AdaStep retains the episode-level signal and adaptively weights the local correction for each step-level group.}
\label{fig:intro}
\vspace{-2.81em}
\end{wrapfigure}

Large language models (LLMs) are increasingly moving from chat-based interaction to \emph{agentic} decision making~\citep{yao2023react,zhou2024archer,wang2026envrl,zheng2026pgpo}. Unlike a chat model, which primarily maps a prompt to a response, an agent repeatedly interacts with an external environment by observing states, taking actions, and receiving subsequent states and rewards. Completing a task therefore requires coordinating a sequence of decisions whose consequences may unfold over a long horizon. Meanwhile, group-based reinforcement learning has become an important paradigm for LLM post-training~\citep{yu2025dapo,min2026dhpo,zhang2026aero,liu2026luspo}, offering a simple, critic-free learning signal. Given multiple rollouts of the same task, the mean return serves as an empirical estimate of the task-level value, and each rollout's return relative to this mean provides an advantage estimate, without requiring a separately trained critic~\citep{shao2024deepseekmath,yu2025dapo}. Under binary success rewards, successful trajectories receive nonnegative advantages and failed trajectories receive nonpositive advantages. Sparse, delayed rewards make it difficult to identify which intermediate actions contributed to an eventual outcome~\citep{arjona2019rudder}. For agentic tasks, however, the long-horizon nature of decision making exposes a fundamental limitation of trajectory-level credit assignment in Group Relative Policy Optimization (GRPO)~\citep{shao2024deepseekmath}. By broadcasting the same outcome-based advantage to every state--action pair along a trajectory, GRPO provides a global but coarse learning signal: it does not distinguish actions that contributed to success from those that were irrelevant or erroneous, nor does it identify which decisions were responsible for failure. Consequently, useful actions in failed trajectories may be penalized, while unhelpful actions in successful trajectories may be reinforced.

To address this coarse credit assignment, Group in GRPO (GiGPO)~\citep{feng2025gigpo} augments the trajectory-level advantage with a local, step-level correction. It groups state--action pairs sharing the same anchor state across rollouts and uses the group's mean discounted return as an empirical state-value baseline. Each sampled return is then compared with this baseline to obtain a step-level relative advantage. This construction provides finer-grained credit assignment while preserving critic-free estimation. However, these local advantages are estimated from finite samples, and their credibility can vary across step-level groups~\citep{li2026when}.
Step-level credit provides finer supervision, but its estimates are
affected by downstream randomness. Even for the same state--action pair,
stochastic transitions, subsequent actions, and discounted reward timing
can produce advantage estimates with opposite signs. Strong local
corrections may therefore amplify estimation noise, whereas weak
corrections may suppress useful action information. This trade-off
connects to variance control in advantage
estimation~\citep{schulman2015gae} and shrinkage-based reward-baseline
estimation with limited rollouts~\citep{zeng2025shrinking}.

GiGPO~\citep{feng2025gigpo} adds a local step-level correction to the global
trajectory-level advantage using a fixed coefficient. However, this
correction can differ in credibility across anchor-state groups.
In some groups, return differences mainly reflect differences in action
quality, making the local estimates useful for guiding policy updates.
In others, much of the return variation arises from subsequent decisions
or stochastic transitions, even when the same action is taken.
Such variation makes the local estimates less credible.
A fixed coefficient cannot adapt to these differences: a large weight
may amplify noisy corrections, while a small weight may weaken useful
local signals. We propose \textbf{AdaStep}, which retains the global
advantage as a reference and adapts the local step-level correction
strength to each group's estimation credibility, as illustrated in
Figure~\ref{fig:intro}. We derive the coefficient under a mean-squared-error
objective and estimate it from existing rollout statistics, requiring
no critic, additional rollouts, or extra model forward passes.

Our main contributions are as follows:
\begin{enumerate}
    \item \textbf{An estimation formulation of local credit assignment.}
    Taking the trajectory-level advantage as a reference, we formulate local advantage weighting as an estimation problem for the true step-level advantage. This formulation makes explicit the trade-off between preserving useful local signals and suppressing sampling error.

    \item \textbf{A principled correction coefficient for each group.}
    Under the proposed objective and stated assumptions, we derive the optimal correction coefficient and obtain a computationally tractable expression that adapts the strength of local correction to the estimation credibility of each step-level group.

    \item \textbf{Empirical gains across three benchmarks.}
    Experiments on three agent-task benchmarks demonstrate that the proposed correction improves performance, supporting the value of accounting for local estimation credibility in group-relative credit assignment. AdaStep tops all trained baselines, gaining up to +9.36 points over GiGPO.

    \item \textbf{Minimal computational overhead.}
    The method requires only a lightweight coefficient computation for each step-level group and integrates directly into the existing advantage calculation with minimal additional time and memory costs, adding about 1\% credit-computation time.
\end{enumerate}

\section{Related Work}

\textbf{Reinforcement learning for large language models.}
Beyond preference alignment through RLHF, reinforcement learning is increasingly used to enhance LLM reasoning and decision making. Group-based methods, including RLOO~\citep{ahmadian2024rloo}, GRPO~\citep{shao2024deepseekmath}, Dr.~GRPO~\citep{liu2025drgrpo}, DAPO~\citep{yu2025dapo}, and CPPO~\citep{lin2025cppo}, provide scalable alternatives to PPO~\citep{schulman2017ppo}. These methods estimate relative advantages from multiple responses to the same prompt, eliminating the need for a separately trained value function and reducing training costs. Recent work further advances this framework through hybrid token--sequence optimization in DHPO~\citep{min2026dhpo}, adaptive rollout selection in AERO~\citep{zhang2026aero}, and length-unbiased sequence-level optimization in LUSPO~\citep{liu2026luspo}, emphasizing efficiency, robustness, and bias reduction in LLM post-training.

\textbf{Step-level advantage assignment.}
GiGPO estimates step-level advantages by grouping state--action pairs with shared anchor states across rollouts. HGPO~\citep{he2026hgpo} and ProGPO~\citep{fan2026progpo} extend this idea through additional grouping mechanisms, introducing further computational costs. Critic-based approaches such as CAPO~\citep{xie2025capo} and SAPO~\citep{liang2026sapo} incur the time and memory costs of training a critic, while EnvRL~\citep{wang2026envrl} adds auxiliary supervision through state prediction and inverse dynamics. PGPO~\citep{zheng2026pgpo} refines step-level advantage signals but does not explicitly address how to control their correction strength. ECPO~\citep{li2026when} constructs heuristically motivated calibrated action-level advantages, whereas AdaStep retains the occurrence-level estimator and derives its state-specific scaling from an explicit estimation objective. Our analysis accounts for the action-conditional bias induced by the shared group baseline, while exploiting the cancellation of this bias after marginalizing over actions.

\section{Preliminaries}
\label{sec:prelim}

\subsection{Problem setup}
\label{sec:setup}

We consider a general setting in which an LLM agent interacts with an environment to
complete multi-step tasks described by a task prompt $x \sim p(X)$. We model the
environment as a Markov decision process (MDP)
$\mathcal{M} = (\mathcal{S}, \mathcal{A}, P, r, \gamma)$, where
$\mathcal{S}$ is the state space of the environment, $\mathcal{A}$ is the action
space, $P(s_{t+1}\mid s_t,a_t)$ is the state
transition model, $r$ is the reward function, and $\gamma$ is the discount factor. At each discrete time
step $t = 1, 2, \dots, T$, the environment is in a state $s_t \in \mathcal{S}$ and
the agent generates a textual action
$a_t \in \mathcal{A}_t$, where in this text world the action space is
$\mathcal{A}_t = \mathcal{V}^{n}$, with $\mathcal{V}$ the token vocabulary and $n$ the
maximum generation length. In text worlds, $\mathcal{S}$ is typically a finite discrete set whose elements consist of text. After executing $a_t$, the environment returns a scalar
reward $r_t \in \mathbb{R}$ and transitions to the next state $s_{t+1}$ according to
$P$. A full
episode consists of a trajectory
$\tau = \big\{ (s_1,a_1,r_1), (s_2,a_2,r_2), \dots,
(s_T,a_T,r_T) \big\}$ with horizon $T$. The agent's behavior is governed by an LLM policy
$\pi_\theta(a_t \mid s_t,x)$ parameterized by $\theta$, which defines a distribution over
textual outputs conditioned on the task prompt $x$ and the current state $s_t$.

In many realistic scenarios the environment provides sparse or delayed rewards (e.g., success and failure signals at the end of an episode). 
Because the agent generates $T$ consecutive textual actions $(a_1, \dots, a_T)$, each potentially spanning thousands of tokens, assigning credit to individual tokens over the course of an episode is particularly challenging.
In this paper we focus on tasks whose quality is measured by a binary \emph{success} signal: the agent receives no reward before the final step ($r_t = 0$ for all $t < T$), and $r_T = 1$ only upon success in the final state, $r_T = 0$ otherwise. 
Therefore, for a trajectory $\tau_i$, the episode return is given by $R_E(\tau_i) = \sum_{t=1}^{T} r_{i,t}$, taking the value $1$ if and only if the trajectory succeeds, and $0$ otherwise.
\subsection{GRPO and GIGPO}
\label{sec:grpo}

GRPO \citep{shao2024deepseekmath} removes the value network that PPO relies on and
instead computes advantages by contrasting returns within a group of rollouts that share
the same prompt. For a prompt $x$ it samples $n$ trajectories $\{\tau_{1}, \tau_{2},\cdots,\tau_n\}$ and sets
\begin{equation}
A_{E}(\tau_i) = \frac{R_E(\tau_i) - \text{mean} \left(  \{R(\tau_j)\}_{j=1}^{N} \right)  }{\operatorname{std}(\{R(\tau_j)\}_{j=1}^{N})}.
\label{eq:grpo-adv}
\end{equation}
In GRPO, the episode return $R_E(\tau_i)$ is the ground-truth signal, and the group-relative advantage $A_{E}$ computed for the trajectory $\tau_i$ is a single scalar applied to the entire trajectory $\tau_i$; hence, for every state-action pair $(\boldsymbol{s}_t ^{(i)}, \boldsymbol{a}_t ^{(i)}) \in \tau_i$, we have $A_E(\boldsymbol{s}_t ^{(i)}, \boldsymbol{a}_t ^{(i)}) = A_E(\tau_i)$. However, this advantage with respect to the state-action pair only captures the global information of the episode, namely, whether it succeeded and is therefore relatively coarse.

GiGPO aims to address the problem that GRPO's signal is overly coarse, seeking fine-grained advantages that capture the contribution of individual actions to success.  For the state $s_{i,t}$ in trajectory $\tau_i$, its state return is the discounted return obtained from that state onward:
$R_S(s_{i,t}) = \sum_{k=t}^{T} \gamma^{k-t} r_{i,k}$
Under the sparse setting, $r_{i,k} = 0$ for all $k < T$, while $r_{i,T} \in \{0,1\}$, so $R_S(s_{i,t}) = \gamma^{T-t} r_{i,T}$.

Specifically, for a state set $\mathcal{S} = \{\tilde{\boldsymbol{s}}_1, \cdots, \tilde{\boldsymbol{s}}_K\}$ of size $K$, on top of GRPO's group rollouts within the same task, GiGPO matches states via a hash map and aggregates identical states together with their subsequent actions and returns. Within a cluster, the same state is referred to as an \emph{anchor state}. The actions and returns immediately following the states in a cluster form a step-level group; that is,
\begin{equation*}
\begin{aligned}
\mathcal{G}(\tilde{\boldsymbol{s}}_k) = \{(\boldsymbol{a}_t ^{(i)}, R_t ^{(i)}): (\boldsymbol{s}_t,\boldsymbol{a}_t ^{(i)}) \in  \tau_i, i \in [1:N], t \in [1:T]\}
\end{aligned}
\end{equation*}
Based on this, the step-level group advantage of an action is computed as
\begin{equation}
  \label{eq:A_GiGPO}
\begin{aligned}
A_{S}(\boldsymbol{s}_t ^{(i)},\boldsymbol{a}_t ^{(i)}) = \frac{R_t ^{(i)} - \text{mean}(  \{R_t ^{(j)} | (\boldsymbol{a}_t ^{(j)}, R_t ^{(j)}) \in \mathcal{G}(\tilde{\boldsymbol{s}}_k) \} )}{\text{std}(  \{R_t ^{(j)} | (\boldsymbol{a}_t ^{(j)}, R_t ^{(j)}) \in \mathcal{G}(\tilde{\boldsymbol{s}}_k) \} )}
\end{aligned}
\end{equation}
Therefore, in GiGPO, the advantage of an action is
\begin{equation*}
\begin{aligned}
A_G(\boldsymbol{s}_t ^{(i)}, \boldsymbol{a}_t ^{(i)}) = A_E(\boldsymbol{s}_t ^{(i)}, \boldsymbol{a}_t ^{(i)}) + \omega A_S (\boldsymbol{s}_t ^{(i)}, \boldsymbol{a}_t ^{(i)})
\end{aligned}
\end{equation*}
where $\omega \geqslant 0$ is the coefficient that controls the overall strength of the step-level relative advantage. In summary, the episode-level relative advantage provides a global reference, while the step-level relative advantage supplies an additional local correction.

\section{Methods}
\label{sec:problem}

\begin{figure}[t]
\centering
\includegraphics[width=\linewidth]{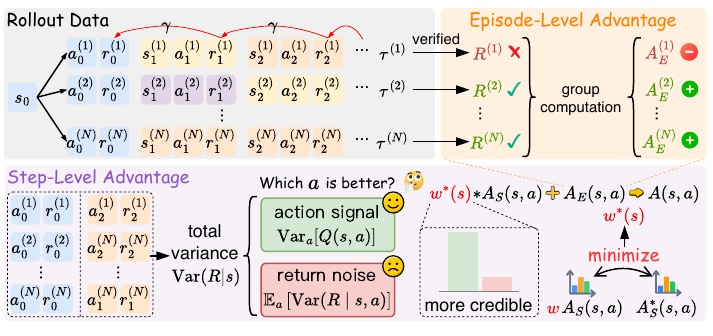}
\caption{Overview of AdaStep. Rollout trajectories provide both episode-level returns and step-level groups. The episode-level relative advantage $A_E$ supplies a global success/failure signal, while the step-level relative advantage $A_S$ provides local credit but is affected by return noise. AdaStep estimates the reliability of $A_S$ at each anchor state $\boldsymbol{s}$ using the fraction of return variance explained by action-dependent value differences, yielding the per-state weight $w^{*}(\boldsymbol{s})$. The final advantage combines the two signals as $A(\boldsymbol{s},\boldsymbol{a}) = A_E(\boldsymbol{s},\boldsymbol{a}) + w^{*}(\boldsymbol{s})A_S(\boldsymbol{s},\boldsymbol{a})$.}
\label{fig:method}
\end{figure}

\subsection{Motivation}

In policy-gradient methods, for a state-action pair with positive advantage, the gradient update increases its probability under the new policy relative to the old policy, and decreases it for pairs with negative advantage. The problem is that GRPO's global trajectory-level signal is too coarse: it reflects only whether the entire trajectory succeeds and carries no information about which decisions caused the task to succeed or fail. GiGPO's results show that adding a local advantage correction to the global advantage does improve performance. However, step-level advantages are based on the state-action return of the current trajectory, which is affected by randomness in the subsequent trajectory and therefore cannot always faithfully reflect the true advantage of a state-action pair.

\begin{example}
A task requires purchasing meat snacks with the \texttt{three meat} flavor and the \texttt{3 pack} option. At the same product-detail-page state, two trajectories both execute \texttt{click[three meat]}, correctly selecting the requested flavor. The shorter trajectory then selects the pack size and completes the purchase in 8 steps. The longer trajectory also selects the correct pack size but returns to the search results and browses additional pages before succeeding in 23 steps. Due to return discounting and within-group comparisons, GiGPO assigns advantages of $3.33$ and $-1.51$ to this same correct action in the shorter and longer trajectories, respectively. This negative advantage is misleading: the action is correct in both cases, and the difference arises from stochastic subsequent behavior rather than from the action itself. Our method assigns advantages of $1.07$ and $0.35$, respectively, correcting the sign in the longer trajectory and reducing the discrepancy between the advantages assigned to the same state--action pair.
\end{example}

\begin{example}
A task requires placing a cooled pot on a stove. At the state described by ``You arrive at sinkbasin 1. On the sinkbasin 1, you see a apple 1, a egg 1, and a lettuce 1'', multiple trajectories execute \texttt{clean pot 1 with sinkbasin 1}. However, cleaning the pot does not cool it. Within one group, this action occurs four times. One trajectory succeeds through subsequent correct actions and receives a positive return, whereas the other three have zero return. GiGPO assigns the cleaning action an advantage of $4.13$ in the successful trajectory and $-0.74$ in each of the other three. This positive credit is misleading because the cleaning action itself does not satisfy the cooling requirement; success results from later actions. Our method changes these advantages to $0.86$ and $-0.56$, respectively.
\end{example}

The local step-level signal is therefore valuable but not fully credible, and how strongly to trust it (i.e., how to control the strength of the step-level correction) is an important question.
Throughout this section we write $\boldsymbol{s}$ for the \emph{anchor state} of GiGPO, the state on which a step-level group $\mathcal{G}(\boldsymbol{s})$ is formed (Section~\ref{sec:grpo}). Figure~\ref{fig:method} gives an overview of AdaStep, and the complete training procedure is summarized in Algorithm~\ref{alg:adastep} (Appendix~\ref{sec:adastep-pseudocode}).

\subsection{Optimal per-state shrinkage}
\label{sec:optimal-shrinkage}

We assume that, before the advantage is normalized by its standard deviation, for each state--action pair in a task, there exists a ground-truth step-level advantage $A_S^{*}(\boldsymbol{s},\boldsymbol{a})$ that we would like to recover but cannot obtain exactly. In practice, within the step-level group $\mathcal{G}(\boldsymbol{s})$ of an anchor state $\boldsymbol{s}$, we estimate it using the unnormalized form of Eq.~\eqref{eq:A_GiGPO}, obtaining a noisy estimate
\begin{equation}
A_S(\boldsymbol{s},\boldsymbol{a}^{(i)}) = A_S^{*}(\boldsymbol{s},\boldsymbol{a}^{(i)}) + \eta_i,
\label{eq:noisy-estimate}
\end{equation}
where $\eta_i$ is the estimation noise associated with the trajectory $\tau_i$ in the group. Because the step-level advantage is computed \emph{within} a group, we seek a per-group weight that rescales the group's estimate so as to best approximate the true advantage. For each anchor state $\boldsymbol{s}$ we therefore solve
\begin{equation}
w^{*}(\boldsymbol{s}) = \arg\min_{w\in\mathbb{R}}\; \mathbb{E}_{\boldsymbol{a}^{(i)}\sim\pi(\cdot\mid\boldsymbol{s})}\!\left[\bigl(w\,A_S(\boldsymbol{s},\boldsymbol{a}^{(i)}) - A_S^{*}(\boldsymbol{s},\boldsymbol{a}^{(i)})\bigr)^{2} \;\middle|\; \boldsymbol{s}\right].
\label{eq:ms-objective}
\end{equation}
This objective asks: among all scalars $w$ applied to the \emph{estimate} $A_S$, which one lands closest, in mean square over the action distribution at $\boldsymbol{s}$, to the \emph{truth} $A_S^{*}$? The solution is a standard projection, requiring only that the second moment of the estimate be strictly positive.

\begin{theorem}[Optimal per-state  weight]
\label{thm:w-star}
Fix an anchor state $\boldsymbol{s}$ and a step-level group $\mathcal{G}(\boldsymbol{s})$ of size $N_s \geqslant 2$ with the noise model in Eq.~\eqref{eq:noisy-estimate}, with a strictly positive second moment of the estimate, $\mathbb{E}[A_S(\boldsymbol{s},\boldsymbol{a})^{2}\mid\boldsymbol{s}] > 0$. Then the minimizer of Eq.~\eqref{eq:ms-objective} is
\begin{equation}
w^{*}(\boldsymbol{s}) = \frac{\mathbb{E}\!\left[A_S(\boldsymbol{s},\boldsymbol{a})\,A_S^{*}(\boldsymbol{s},\boldsymbol{a})\,\middle|\,\boldsymbol{s}\right]}{\mathbb{E}\!\left[A_S(\boldsymbol{s},\boldsymbol{a})^{2}\,\middle|\,\boldsymbol{s}\right]}.
\label{eq:w-star}
\end{equation}
\end{theorem}

\noindent The numerator measures how well the estimate $A_S$ aligns with the truth $A_S^{*}$; the denominator is the second moment of the estimate and is the sole condition needed for the solution to be well defined. The proof is deferred to Appendix~\ref{sec:proofs}. Theorem~\ref{thm:w-star} is a population-level projection: it is exact but, as written, \emph{not yet computable}, since it requires the unknown $A_S^{*}$. We next replace it by an equivalent form that depends only on group statistics GiGPO already collects.

\subsection{Computable form from group statistics}
\label{sec:computable}

To make the optimal weight in Eq.~\eqref{eq:w-star} tractable, we use the definition of step-level group relative advantage in Eq.~\eqref{eq:A_GiGPO}. Within a group $\mathcal{G}(\boldsymbol{s})$ of size $N_s$, the step-level advantage without standard-deviation normalization is:
\begin{equation}
A_S(\boldsymbol{s},\boldsymbol{a}^{(i)}) = R^{(i)} - \hat{V}(\boldsymbol{s}), \qquad \hat{V}_{N_s}(\boldsymbol{s}) = \frac{1}{N_s}\sum_{j\in\mathcal{G}(\boldsymbol{s})} R^{(j)},
\label{eq:step-adv}
\end{equation}
while, under a fixed policy $\pi$, the ground truth it estimates is
\begin{equation}
A_S^{*}(\boldsymbol{s},\boldsymbol{a}^{(i)}) = Q(\boldsymbol{s},\boldsymbol{a}^{(i)}) - V(\boldsymbol{s}), \qquad V(\boldsymbol{s}) = \mathbb{E}_{\boldsymbol{a} \sim \pi(\cdot|\boldsymbol{s})}[Q(\boldsymbol{s},\boldsymbol{a})\mid\boldsymbol{s}] = \mathbb{E}_{\pi}[R\mid\boldsymbol{s}],
\label{eq:true-adv}
\end{equation}
Therefore, in the computation, GiGPO uses the mean of the returns within a group as the estimate of the state-value function, and only a single return as the estimate of the state-action value, where $Q(\boldsymbol{s},\boldsymbol{a}) = \mathbb{E}_{\pi}[R\mid\boldsymbol{s},\boldsymbol{a}]$ is the action-value function. Therefore, the estimation noise of the advantage arises from two sources: the estimation of $Q$ and the estimation of $V$. Furthermore, by computing the estimation error $\eta_i = A_S(\boldsymbol{s},\boldsymbol{a}^{(i)}) - A_S^{*}(\boldsymbol{s},\boldsymbol{a}^{(i)})$, we obtain
\begin{equation}
\eta_i = \frac{N_s-1}{N_s}\bigl(R^{(i)} - Q(\boldsymbol{s},\boldsymbol{a}^{(i)})\bigr) \;-\; \frac{A_S^{*}(\boldsymbol{s},\boldsymbol{a}^{(i)})}{N_s} \;-\; \frac{1}{N_s}\sum_{j\neq i}\bigl(R^{(j)} - V(\boldsymbol{s})\bigr).
\label{eq:eta-decomp}
\end{equation}
The left term is the sampling noise of the current rollout---the realized return minus its conditional mean; the middle term is a $1/N_s$ shrinkage of the true step advantage, an artifact of estimating the baseline with a finite group that includes sample $i$ itself; the right term is the average deviation of the \emph{other} group members from the state value. Taking conditional expectations, the left and right terms vanish and only the middle term survives: $\mathbb{E}_{i}[\eta_i\mid\boldsymbol{s},\boldsymbol{a}^{(i)}] = -A_S^{*}(\boldsymbol{s},\boldsymbol{a}^{(i)})/N_s.$
Conditioned on both $\boldsymbol{s}$ and the action, the estimate carries a $1/N_s$ bias; and the marginalize over the action this bias disappears, since $\mathbb{E}_{\boldsymbol{a}}[A_S^{*}\mid\boldsymbol{s}] = 0$ by definition of $V(\boldsymbol{s})$,$
\mathbb{E}_{i,\boldsymbol{a}\sim\pi(\cdot\mid\boldsymbol{s})}[\eta_i\mid\boldsymbol{s}] = 0.$
Thus within a step group the step-level estimate is \emph{unbiased} for the true step advantage---reweighting, rather than sign-flipping or discarding, is the right intervention.

Furthermore, we assume that within a group $\mathcal{G}(\boldsymbol{s})$, the action--return pairs $(\boldsymbol{a}^{(j)},R^{(j)})$ are independent and identically distributed conditional on $\boldsymbol{s}$, and that the return variance within a group is finite. Under these two practical assumptions, the numerator and denominator of the optimal weight in Eq.~\eqref{eq:w-star} admit the following closed forms:

\begin{lemma}
\label{lem:closed-form}
Assume the action--return pairs $(\boldsymbol{a}^{(j)},R^{(j)})$ are i.i.d.\ conditional on $\boldsymbol{s}$ with finite return variance, Then the numerator and denominator of~\eqref{eq:w-star} admit the closed forms
\begin{equation*}
\begin{aligned}
\mathbb{E}[A_S\,A_S^{*}\mid\boldsymbol{s}] = \frac{N_s-1}{N_s}\,\operatorname{Var}_{\boldsymbol{a}\sim\pi(\cdot\mid\boldsymbol{s})}\!\bigl[Q(\boldsymbol{s},\boldsymbol{a})\bigr], 
\mathbb{E}[A_S^{2}\mid\boldsymbol{s}] = \frac{N_s-1}{N_s}\,\operatorname{Var}(R\mid\boldsymbol{s}).
\end{aligned}
\end{equation*}
\end{lemma}
\noindent The common factor $(N_s-1)/N_s$ cancels in the ratio. A restatement and proof are given in Lemma~\ref{lem:appendix-closed-form} of Appendix~\ref{sec:proofs}.

We first express the optimal coefficient as a signal-to-(signal-plus-noise) ratio to make its credibility interpretation explicit. By the law of total variance, at a fixed state $\boldsymbol{s}$, the return variance decomposes into the variance of $Q(\boldsymbol{s},\boldsymbol{a})$ across actions $\boldsymbol{a}\sim\pi(\cdot\mid\boldsymbol{s})$ and the conditional return variance after fixing both $\boldsymbol{s}$ and $\boldsymbol{a}$, averaged over the same action distribution:
\begin{equation}
\operatorname{Var}(R\mid\boldsymbol{s}) = \operatorname{Var}_{\boldsymbol{a}\sim\pi(\cdot\mid\boldsymbol{s})}\!\bigl[Q(\boldsymbol{s},\boldsymbol{a})\bigr] + \mathbb{E}_{\boldsymbol{a}\sim\pi(\cdot\mid\boldsymbol{s})}\!\bigl[\operatorname{Var}(R\mid\boldsymbol{s},\boldsymbol{a})\bigr].
\label{eq:tva}
\end{equation}
The decomposition is proved in Lemma~\ref{lem:return-variance} of Appendix~\ref{sec:proofs}.

\begin{table}[t]
\centering
\caption{\textbf{Main results.} Performance across three benchmarks. Results are averaged over $3$ random seeds. The best and second-best values within each model block are shown in \textbf{bold} and \underline{underlined}, respectively.
$\Delta$ vs.\ GiGPO reports AdaStep minus GiGPO for the same model.}
\label{tab:main}
\footnotesize
\setlength{\tabcolsep}{1.0pt}
\renewcommand{\arraystretch}{1.08}
\begin{tabularx}{\textwidth}{@{}l>{\raggedright\arraybackslash}Xccccc@{}}
\toprule
\multirow{2}{*}{\textbf{Type}} & \multirow{2}{*}{\textbf{Method}}
& \multicolumn{2}{c}{\textbf{ALFWorld}} & \multicolumn{2}{c}{\textbf{WebShop}}
& \textbf{ScienceWorld} \\
\cmidrule(lr){3-4}\cmidrule(lr){5-6}\cmidrule(lr){7-7}
& & In-Success & Out-Success & Test Score & Success & Success \\
\midrule
\rowcolor{benchmarkgroup}
\multicolumn{7}{@{}l}{\textbf{\textit{Frontier Models}}} \\
Prompting & DeepSeek-V4.1-Flash (ReAct) & $80.60{\scriptstyle\pm 0.2}$ &  $88.83{\scriptstyle\pm 1.9}$ & $28.54{\scriptstyle\pm 1.8}$ & $25.13{\scriptstyle\pm 1.5}$ & $58.15{\scriptstyle\pm 4.3}$ \\
Prompting & GLM-5.2 (ReAct) & $96.36{\scriptstyle\pm 0.9}$ & $95.05{\scriptstyle\pm 1.2}$ & $28.62{\scriptstyle\pm 2.1}$ & $23.47{\scriptstyle\pm 1.8}$ & $82.03{\scriptstyle\pm 1.4}$ \\
\midrule
\rowcolor{benchmarkgroup}
\multicolumn{7}{@{}l}{\textbf{\textit{Model: Qwen3-1.7B}}} \\
Prompting & Pure model &  $9.64{\scriptstyle\pm 1.2}$ &  $11.46{\scriptstyle\pm 2.4}$ & $53.38{\scriptstyle\pm3.1}$ & $5.34{\scriptstyle\pm0.8}$ & $9.38{\scriptstyle\pm2.1}$ \\
Prompting & ReAct~\citep{yao2023react} &  $19.79{\scriptstyle\pm 1.6}$ & $20.05{\scriptstyle\pm 1.8}$ & $33.09{\scriptstyle\pm0.8}$ & $8.72{\scriptstyle\pm1.6}$ & $9.90{\scriptstyle\pm1.6}$\\
Prompting & Reflexion~\citep{shinn2023reflexion} & $52.02{\scriptstyle\pm 1.9}$ & $51.30{\scriptstyle\pm 3.1}$ & $37.26{\scriptstyle\pm1.8}$ & $22.66{\scriptstyle\pm3.6}$ & $18.49{\scriptstyle\pm0.7}$ \\
RL Training & PPO~\citep{schulman2017ppo} & $45.57{\scriptstyle\pm1.6}$ & $42.44{\scriptstyle\pm3.2}$ & $76.77{\scriptstyle\pm3.7}$ & $61.10{\scriptstyle\pm4.4}$ & $21.61{\scriptstyle\pm5.5}$ \\
RL Training & RLOO~\citep{ahmadian2024rloo} & $66.41{\scriptstyle\pm8.2}$ & $53.90{\scriptstyle\pm2.1}$ & $63.23{\scriptstyle\pm0.4}$ & $37.60{\scriptstyle\pm3.6}$ & $18.23{\scriptstyle\pm1.2}$ \\
RL Training & GRPO~\citep{shao2024deepseekmath} & $71.09{\scriptstyle\pm0.8}$ & $\underline{64.58}{\scriptstyle\pm4.3}$ & $71.00{\scriptstyle\pm0.6}$ & $38.83{\scriptstyle\pm3.1}$ & $19.79{\scriptstyle\pm1.2}$ \\
RL Training & GiGPO~\citep{feng2025gigpo} & $71.10{\scriptstyle\pm3.4}$ & $64.06{\scriptstyle\pm0.0}$ & $\underline{80.61}{\scriptstyle\pm0.8}$ & $\underline{64.58}{\scriptstyle\pm0.8}$ & $20.57{\scriptstyle\pm1.8}$ \\
RL Training & HGPO~\citep{he2026hgpo} & $\underline{71.35}{\scriptstyle\pm3.9}$ & $63.80{\scriptstyle\pm1.8}$ & $75.56{\scriptstyle\pm2.6}$ & $43.35{\scriptstyle\pm2.7}$ & $\underline{25.78}{\scriptstyle\pm1.4}$ \\
\rowcolor{adastephighlight}
RL Training & \textbf{AdaStep (Ours)} & $\mathbf{73.70}{\scriptstyle\pm3.5}$ &  $\mathbf{65.79}{\scriptstyle\pm3.0}$  & $\mathbf{81.36}{\scriptstyle\pm2.5}$ & $\mathbf{65.24}{\scriptstyle\pm2.2}$ & $\mathbf{29.93}{\scriptstyle\pm1.9}$ \\
\multicolumn{2}{@{}l}{\textcolor{green!50!black}{\textit{$\Delta$ vs.\ GiGPO}}} & \textcolor{green!50!black}{$+2.60$} & \textcolor{green!50!black}{$+1.73$} & \textcolor{green!50!black}{$+0.75$} & \textcolor{green!50!black}{$+0.66$} & \textcolor{green!50!black}{$+9.3$6} \\
\midrule
\rowcolor{benchmarkgroup}
\multicolumn{7}{@{}l}{\textbf{\textit{Model: Qwen3-4B}}} \\
Prompting & Pure model & $23.96{\scriptstyle\pm 3.1}$ & $17.19{\scriptstyle\pm 2.0}$ & $25.70{\scriptstyle\pm1.7}$ & $1.95{\scriptstyle\pm0.3}$ & $14.59{\scriptstyle\pm2.5}$ \\
Prompting & ReAct~\citep{yao2023react} & $77.61{\scriptstyle\pm 1.6}$ & $80.99{\scriptstyle\pm 2.5}$ & $29.55{\scriptstyle\pm0.6}$ & $16.93{\scriptstyle\pm0.9}$ & $15.10{\scriptstyle\pm1.9}$ \\
Prompting & Reflexion~\citep{shinn2023reflexion} & $79.95{\scriptstyle\pm 4.0}$ &  $74.42{\scriptstyle\pm 1.5}$ & $35.87{\scriptstyle\pm1.3}$ & $29.04{\scriptstyle\pm1.4}$ & $37.76{\scriptstyle\pm1.1}$ \\
RL Training & PPO~\citep{schulman2017ppo} & $38.54{\scriptstyle\pm1.6}$ & $40.63{\scriptstyle\pm3.6}$ & $74.23{\scriptstyle\pm4.4}$ & $61.37{\scriptstyle\pm2.1}$ & $30.97{\scriptstyle\pm0.4}$ \\
RL Training & RLOO~\citep{ahmadian2024rloo} & $76.04{\scriptstyle\pm0.9}$ & $76.82{\scriptstyle\pm0.5}$ & $79.02{\scriptstyle\pm1.9}$ & $67.58{\scriptstyle\pm2.4}$ & $37.76{\scriptstyle\pm1.7}$ \\
RL Training & GRPO~\citep{shao2024deepseekmath} & $77.61{\scriptstyle\pm3.7}$ & $73.96{\scriptstyle\pm2.4}$ & $83.60{\scriptstyle\pm1.6}$ & $70.43{\scriptstyle\pm0.2}$ & $42.97{\scriptstyle\pm2.7}$ \\
RL Training & GiGPO~\citep{feng2025gigpo} & $88.02{\scriptstyle\pm0.9}$ & $83.07{\scriptstyle\pm4.6}$ & $87.34{\scriptstyle\pm0.6}$ & $\underline{78.28}{\scriptstyle\pm1.8}$ & $46.67{\scriptstyle\pm2.3}$ \\
RL Training & HGPO~\citep{he2026hgpo} & $\underline{91.93}{\scriptstyle\pm1.6}$ & $\underline{85.42}{\scriptstyle\pm0.5}$ & $\underline{88.18}{\scriptstyle\pm0.4}$ & $76.43{\scriptstyle\pm1.6}$ & $\underline{47.63}{\scriptstyle\pm1.3}$ \\
\rowcolor{adastephighlight}
RL Training & \textbf{AdaStep (Ours)} & $\mathbf{93.01}{\scriptstyle\pm0.9}$ & $\mathbf{86.58}{\scriptstyle\pm1.2}$ & $\mathbf{88.76}{\scriptstyle\pm0.5}$ & $\mathbf{80.27}{\scriptstyle\pm0.4}$ & $\mathbf{48.70}{\scriptstyle\pm2.7}$ \\
\multicolumn{2}{@{}l}{\textcolor{green!50!black}{\textit{$\Delta$ vs.\ GiGPO}}} & \textcolor{green!50!black}{$+4.99$} & \textcolor{green!50!black}{$+3.51$} & \textcolor{green!50!black}{$+1.42$} & \textcolor{green!50!black}{$+1.99$} & \textcolor{green!50!black}{$+2.03$} \\
\midrule
\rowcolor{benchmarkgroup}
\multicolumn{7}{@{}l}{\textbf{\textit{Model: Qwen2.5-7B-Instruct}}} \\
Prompting & Pure model & $14.80$ & - & $26.40$ & $7.80$ & $13.00$ \\
Prompting & ReAct~\citep{yao2023react} & $31.20$ & - & $46.20$ & $19.50$ & $12.50$ \\
Prompting & Reflexion~\citep{shinn2023reflexion} & $42.70$ & - & $58.10$ & $28.80$ & $23.70$ \\
RL Training & PPO~\citep{schulman2017ppo} & $43.23{\scriptstyle\pm1.8}$ & $42.19{\scriptstyle\pm4.8}$ & $81.40{\scriptstyle\pm3.1}$ & $68.7{\scriptstyle\pm5.1}$ & $24.90{\scriptstyle\pm3.6}$ \\
RL Training & RLOO~\citep{ahmadian2024rloo} & $73.70{\scriptstyle\pm1.2}$ & $67.19{\scriptstyle\pm0.8}$ & $80.3{\scriptstyle\pm3.2}$ & $65.7{\scriptstyle\pm4.0}$ & $37.70{\scriptstyle\pm1.0}$ \\
RL Training & GRPO~\citep{shao2024deepseekmath} & $80.47{\scriptstyle\pm1.6}$ & $80.99{\scriptstyle\pm0.9}$ & $79.3{\scriptstyle\pm2.8}$ & $66.1{\scriptstyle\pm3.7}$ & $42.73{\scriptstyle\pm2.5}$ \\
RL Training & GiGPO~\citep{feng2025gigpo} & $92.71{\scriptstyle\pm0.9}$ & $89.06{\scriptstyle\pm1.4}$ & $88.93{\scriptstyle\pm1.5}$ & $77.60{\scriptstyle\pm1.7}$ & $46.63{\scriptstyle\pm1.2}$ \\
RL Training & HGPO~\citep{he2026hgpo} & $\underline{94.79}{\scriptstyle\pm0.5}$ & $\underline{93.23}{\scriptstyle\pm1.2}$ & $\underline{88.96}{\scriptstyle\pm1.0}$ & $\underline{78.51}{\scriptstyle\pm1.4}$ & $\underline{46.87}{\scriptstyle\pm1.5}$ \\
\rowcolor{adastephighlight}
RL Training & \textbf{AdaStep (Ours)} & $\mathbf{97.66}{\scriptstyle\pm0.8}$ & $\mathbf{93.30}{\scriptstyle\pm1.3}$ & $\mathbf{90.46}{\scriptstyle\pm0.4}$ & $\mathbf{82.42}{\scriptstyle\pm1.2}$ & $\mathbf{47.17}{\scriptstyle\pm0.5}$ \\
\multicolumn{2}{@{}l}{\textcolor{green!50!black}{\textit{$\Delta$ vs.\ GiGPO}}} & \textcolor{green!50!black}{$+4.95$} & \textcolor{green!50!black}{$+4.24$} & \textcolor{green!50!black}{$+1.53$} & \textcolor{green!50!black}{$+4.82$} & \textcolor{green!50!black}{$+0.54$} \\
\bottomrule
\end{tabularx}
\end{table}

\begin{theorem}
\label{thm:w-computable}
Under the assumptions of Theorem~\ref{thm:w-star} and Lemma~\ref{lem:closed-form}, the optimal per-state weight~\eqref{eq:w-star} admits two equivalent forms. The first expresses the weight as the fraction of return variance due to differences in expected return across actions:
\begin{equation}
w^{*}(\boldsymbol{s}) = \frac{\operatorname{Var}_{\boldsymbol{a}\sim\pi(\cdot\mid\boldsymbol{s})}\!\bigl[Q(\boldsymbol{s},\boldsymbol{a})\bigr]}{\operatorname{Var}_{\boldsymbol{a}\sim\pi(\cdot\mid\boldsymbol{s})}\!\bigl[Q(\boldsymbol{s},\boldsymbol{a})\bigr] + \mathbb{E}_{\boldsymbol{a}\sim\pi(\cdot\mid\boldsymbol{s})}\!\bigl[\operatorname{Var}(R\mid\boldsymbol{s},\boldsymbol{a})\bigr]}\;\in[0,1].
\label{eq:w-signal}
\end{equation}
By the law of total variance, the equivalent computable form in terms of return variances is
\begin{equation}
w^{*}(\boldsymbol{s}) = 1 - \frac{\mathbb{E}_{\boldsymbol{a}\sim\pi(\cdot\mid\boldsymbol{s})}\!\bigl[\operatorname{Var}(R\mid\boldsymbol{s},\boldsymbol{a})\bigr]}{\operatorname{Var}(R\mid\boldsymbol{s})}.
\label{eq:w-computable}
\end{equation}
\end{theorem}
See Theorem~\ref{thm:appendix-w-computable} in Appendix~\ref{sec:proofs} for the proof. Equation~\eqref{eq:w-signal} expresses the weight as the fraction of return variance explained by differences in expected return across actions. We regard these differences as signal, and the remaining variability for fixed state--action pairs, averaged over actions, as noise. The weight approaches $1$ when signal dominates and $0$ when noise dominates. This group-specific weight replaces the global coefficient $\omega$ in GiGPO.

Equation~\eqref{eq:w-computable} allows us to estimate the weight from rollout return statistics without explicitly estimating action values. Section~\ref{sec:fallback} specifies the sample requirements and fallback weights for groups with insufficient data.

We compute $w^*(\boldsymbol{s})$ from the unnormalized returns and multiply the local advantage by this weight. If standard-deviation normalization is used, we first normalize the local advantage within $\mathcal{G}(\boldsymbol{s})$ and then apply the same weight.

\section{Experiments}
\label{sec:experiments}

We evaluate AdaStep using three backbone models, Qwen3-1.7B, Qwen3-4B, and
Qwen2.5-7B-Instruct, on three benchmarks: ALFWorld, WebShop, and ScienceWorld.
We compare against closed-source model baselines, prompting methods such as
ReAct, and reinforcement learning algorithms including GRPO, GiGPO, and HGPO.
Detailed benchmark descriptions, baselines, and implementation settings are
provided in Appendix~\ref{sec:experimental-details}.

\subsection{Main results}
\label{sec:main-results}

Table~\ref{tab:main} compares AdaStep with prompting-based and RL baselines across three models and three benchmarks. AdaStep consistently outperforms HGPO and GiGPO across all model--benchmark combinations, with more pronounced gains over GRPO and PPO. These results support the effectiveness of reliability-aware step-level credit assignment across different model scales and task environments.

\subsection{Robustness to advantage normalization}
\label{sec:normalization}

\begin{wraptable}{r}{0.55\textwidth}
\vspace{-\intextsep}
\setlength{\abovecaptionskip}{0pt}
\centering
\footnotesize
\caption{Results with and without std normalization. $\Delta$ denotes gains over GiGPO and best results \textbf{bolded}.}
\label{tab:normalization}
\scriptsize
\setlength{\tabcolsep}{2pt}
\renewcommand{\arraystretch}{1.08}
\begin{tabular}{@{}llcccc@{}}
\toprule
\multirow{2}{*}{\textbf{Std}} & \multirow{2}{*}{\textbf{Method}}
& \multicolumn{2}{c}{\textbf{ALFWorld}}
& \multicolumn{2}{c}{\textbf{WebShop}} \\
\cmidrule(lr){3-4}\cmidrule(lr){5-6}
& & In-Suc. & Out-Suc. & Score & Suc. \\
\midrule
\rowcolor{benchmarkgroup}
\multicolumn{6}{@{}l}{\textbf{\textit{Model: Qwen3-4B}}} \\
w/ & GiGPO & $88.02{\scriptstyle\pm0.9}$ & $83.07{\scriptstyle\pm4.6}$ & $88.77{\scriptstyle\pm0.3}$ & $77.17{\scriptstyle\pm0.3}$ \\
\rowcolor{adastephighlight}
w/ & \textbf{AdaStep} & $\mathbf{93.01}{\scriptstyle\pm0.9}$ & $\mathbf{86.58}{\scriptstyle\pm1.2}$ & $\mathbf{89.90}{\scriptstyle\pm1.1}$ & $\mathbf{79.27}{\scriptstyle\pm0.5}$ \\
\multicolumn{2}{@{}l}{\textcolor{green!50!black}{\textit{$\Delta$}}} & \textcolor{green!50!black}{$+4.99$} & \textcolor{green!50!black}{$+3.51$} & \textcolor{green!50!black}{$+1.13$} & \textcolor{green!50!black}{$+2.10$} \\
\addlinespace[2pt]
w/o & GiGPO & $88.20{\scriptstyle\pm1.2}$ & $85.17{\scriptstyle\pm0.7}$ & $87.34{\scriptstyle\pm0.6}$ & $78.28{\scriptstyle\pm1.8}$ \\
\rowcolor{adastephighlight}
w/o & \textbf{AdaStep} & $\mathbf{91.73}{\scriptstyle\pm1.8}$ & $\mathbf{85.40}{\scriptstyle\pm0.8}$ & $\mathbf{88.76}{\scriptstyle\pm0.5}$ & $\mathbf{80.27}{\scriptstyle\pm0.4}$ \\
\multicolumn{2}{@{}l}{\textcolor{green!50!black}{\textit{$\Delta$}}} & \textcolor{green!50!black}{$+3.53$} & \textcolor{green!50!black}{$+0.23$} & \textcolor{green!50!black}{$+1.42$} & \textcolor{green!50!black}{$+1.99$} \\
\midrule
\rowcolor{benchmarkgroup}
\multicolumn{6}{@{}l}{\textbf{\textit{Model: Qwen2.5-7B-Instruct}}} \\
w/ & GiGPO & $92.71{\scriptstyle\pm0.9}$ & $89.06{\scriptstyle\pm1.4}$ & $88.70{\scriptstyle\pm2.9}$ & $76.53{\scriptstyle\pm 2.7}$ \\
\rowcolor{adastephighlight}
w/ & \textbf{AdaStep} & $\mathbf{97.66}{\scriptstyle\pm0.8}$ & $\mathbf{93.30}{\scriptstyle\pm1.3}$ & $\mathbf{89.93}{\scriptstyle\pm0.1}$ & $\mathbf{79.97}{\scriptstyle\pm0.6}$ \\
\multicolumn{2}{@{}l}{\textcolor{green!50!black}{\textit{$\Delta$}}} & \textcolor{green!50!black}{$+4.95$} & \textcolor{green!50!black}{$+4.24$} & \textcolor{green!50!black}{$+1.23$} & \textcolor{green!50!black}{$+3.44$} \\
\addlinespace[2pt]
w/o & GiGPO & $92.93{\scriptstyle\pm0.3}$ & $88.00{\scriptstyle\pm1.6}$ & $88.93{\scriptstyle\pm1.5}$ & $77.60{\scriptstyle\pm1.7}$ \\
\rowcolor{adastephighlight}
w/o & \textbf{AdaStep} & $\mathbf{96.00}{\scriptstyle\pm0.1}$ & $\mathbf{90.63}{\scriptstyle\pm1.0}$ & $\mathbf{90.46}{\scriptstyle\pm0.4}$ & $\mathbf{82.42}{\scriptstyle\pm1.2}$ \\
\multicolumn{2}{@{}l}{\textcolor{green!50!black}{\textit{$\Delta$}}} & \textcolor{green!50!black}{$+3.07$} & \textcolor{green!50!black}{$+2.63$} & \textcolor{green!50!black}{$+1.53$} & \textcolor{green!50!black}{$+4.82$} \\
\bottomrule
\end{tabular}
\vspace{-3em}
\end{wraptable}

To examine whether AdaStep's gains depend on a particular normalization setting,
Table~\ref{tab:normalization} compares GiGPO and AdaStep both with and without
standard-deviation normalization on Qwen3-4B and Qwen2.5-7B-Instruct. We add reruns of ALFWorld without std
normalization and WebShop with std normalization. In both settings, the
coefficient is computed from unnormalized group statistics before any
normalization of the step-level advantages.

AdaStep consistently achieves higher mean performance than GiGPO both
with and without standard-deviation normalization
(Table~\ref{tab:normalization}), supporting its effectiveness under
both settings.

\subsection{Fallback weights for non-estimable groups}
\label{sec:fallback}
In practice, some step-level groups contain only one distinct action,
only one return observation per action, or identical returns, as
illustrated in Table~\ref{tab:fallback}(a). We therefore estimate the
coefficient in Theorem~\ref{thm:w-computable} only for groups containing
at least two distinct actions, at least two return observations per
action, and positive total return variance.

Groups with identical returns have zero step-level advantages and require
no correction. For other groups that do not meet the estimation criteria,
we compare three default weights: $w=0$, $w=\bar{w}$, and $w=1$, where
$\bar{w}$ is the mean coefficient over eligible groups in the current batch
when such groups exist. Eligible groups use their estimated coefficients
in all three variants.

On WebShop and ALFWorld with Qwen2.5-7B-Instruct, $w=1$ achieves the highest
observed success rate among the three choices
(Table~\ref{tab:fallback}(b)), supporting retention of the original
step-level signal when evidence is insufficient for coefficient estimation.

\begin{table}[t]
\vspace{-1em}
\centering
\caption{
Left: groups outside the estimation criteria; right: Qwen2.5-7B-Instruct
success rates (\%) under different default weights, with the best results in \textbf{bold}.}
\label{tab:fallback}
\footnotesize
\setlength{\tabcolsep}{3pt}
\renewcommand{\arraystretch}{1.15}
\begin{minipage}[t]{0.58\textwidth}
\vspace{0pt}
\centering
\textbf{(a) Group configurations}\par\smallskip
\begin{tabularx}{\linewidth}{@{}>{\raggedright\arraybackslash}X>{\raggedright\arraybackslash}X@{}}
\toprule
\textbf{Observed group} & \textbf{Handling} \\
\midrule
One action; identical returns & $A_S=0$; no correction. \\
\addlinespace[2pt]
One action; varying returns & Cannot compute $\mathbb{E}_{\boldsymbol{a}}[\cdot]$. \\
\addlinespace[2pt]
Multiple actions; one return observation per action & Cannot compute $\operatorname{Var}(R\mid\boldsymbol{s},\boldsymbol{a})$. \\
\bottomrule
\end{tabularx}
\end{minipage}\hfill
\begin{minipage}[t]{0.39\textwidth}
\vspace{0pt}
\centering
\textbf{(b) Fallback ablation}\par\smallskip
\begin{tabular*}{\linewidth}{@{\extracolsep{\fill}}lcc@{}}
\toprule
\textbf{Fallback} & \textbf{WebShop} & \textbf{ALFWorld} \\
\midrule
$w=0$ & $80.1$ & $92.2$ \\
$w=\bar{w}$ (mean) & $77.0$ & $96.1$ \\
\rowcolor{adastephighlight}
$w=1$ & $\mathbf{82.4}$ & $\mathbf{97.6}$ \\
\bottomrule
\end{tabular*}
\end{minipage}
\end{table}

\subsection{Sensitivity to the step-level coefficient}
\label{sec:weight-sensitivity}

\begin{figure}[t]
\centering
\includegraphics[width=\linewidth]{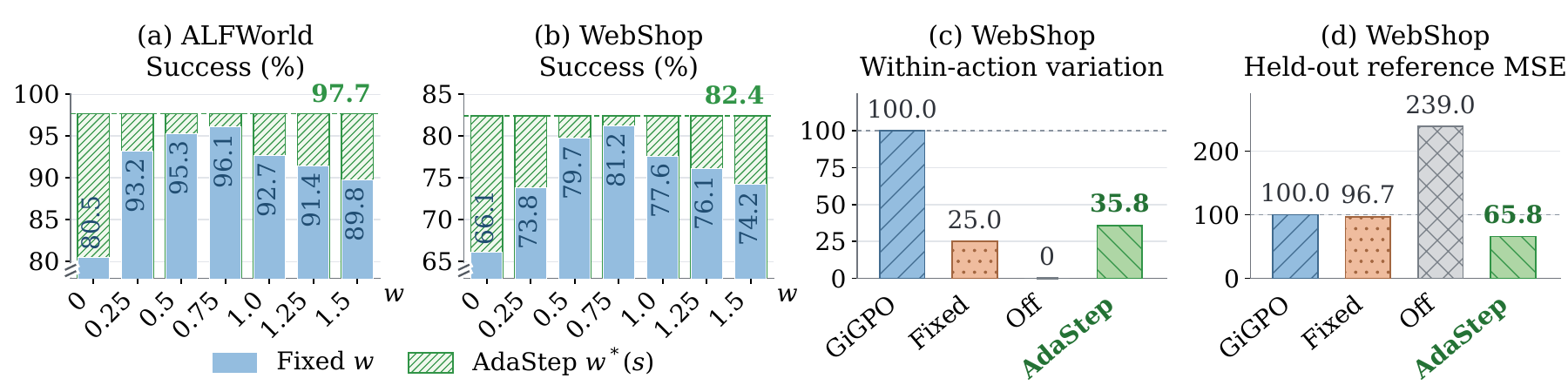}
\caption{\textbf{Coefficient sensitivity and step-credit diagnostics.}
(a,b) Success rates for fixed weights (blue) and AdaStep (hatched green).
(c,d) WebShop within-action variation and held-out reference MSE, each
normalized separately to GiGPO $=100$; Fixed uses $c=0.5$ and Off uses $c=0$.
Appendix~\ref{sec:credit-diagnostic-details} gives the diagnostic protocol.}
\label{fig:weight-sensitivity}
\label{fig:webshop-credit-diagnostics}
\end{figure}
Figure~\ref{fig:weight-sensitivity}(a,b) compares fixed global weights in
$A=A_E+wA_S$ with AdaStep using Qwen2.5-7B-Instruct on ALFWorld and WebShop.
On both benchmarks, moderate local weighting improves performance,
whereas excessive weighting degrades it. AdaStep achieves higher success
rates than every tested fixed coefficient, supporting group-specific
adaptation beyond global weight tuning.

\subsection{Analysis of step-credit estimation on WebShop}
\label{sec:credit-diagnostics}

We examine step-credit estimation on 3,666 historical WebShop anchor-state
groups using two complementary diagnostics. Within-action variation measures
the remaining fluctuation among weighted credits assigned to the same action.
Held-out reference MSE measures their squared deviation from action-level
advantages estimated using a separate subset of samples. Appendix~\ref{sec:credit-diagnostic-details}
describes the coefficient estimation, sample splitting, and metric definitions.

Figure~\ref{fig:webshop-credit-diagnostics}(c,d) shows that AdaStep reduces
within-action variation while achieving the lowest held-out reference MSE.
This suggests a better balance between suppressing fluctuations and
preserving action-level information, as reducing variation alone does
not guarantee better estimation.

\subsection{Computational efficiency}
\label{sec:efficiency}

\begin{wraptable}{r}{0.4\textwidth}
\vspace{-\intextsep}
\setlength{\abovecaptionskip}{0pt}
\centering
\caption{Credit-computation time per iteration and success rates on WebShop.}
\label{tab:efficiency}
\small
\setlength{\tabcolsep}{3pt}
\renewcommand{\arraystretch}{1.12}
\begin{tabularx}{\linewidth}{@{}Xcc@{}}
\toprule
\textbf{Method} & \textbf{Time (s)} & \textbf{Success (\%)} \\
\midrule
GiGPO & $0.4968{\scriptstyle\pm0.0324}$ & $77.60{\scriptstyle\pm1.7}$ \\
HGPO & $0.6340{\scriptstyle\pm0.0711}$ & $78.51{\scriptstyle\pm1.4}$ \\
\rowcolor{adastephighlight}
\textbf{AdaStep} & $0.5018{\scriptstyle\pm0.0114}$ & $\mathbf{82.42}{\scriptstyle\pm1.2}$ \\
\bottomrule
\end{tabularx}
\vspace{-2em}
\end{wraptable}

Table~\ref{tab:efficiency} compares per-iteration return and advantage
computation on WebShop with Qwen2.5-7B-Instruct. AdaStep achieves the highest
mean success rate with approximately $1.0\%$ more credit-computation time
than GiGPO, while also requiring less time than HGPO. These results show
that adaptive step-credit weighting improves performance with minimal
additional computation.

\section{Conclusion}
\label{sec:conclusion}

We introduced AdaStep, which derives an optimal per-state shrinkage coefficient from a mean-squared-error objective on the latent step advantage and reads it as the fraction of return variance explained by the selected action.
Across three backbones and three benchmarks, AdaStep consistently improves over GRPO, GiGPO, and HGPO, with gains that persist with and without advantage normalization, at negligible computational cost.
Extending AdaStep to continuous or partially observable environments, where anchor states cannot be matched directly, remains open and is discussed in Appendix~\ref{sec:limitations}.

\bibliography{references}
\bibliographystyle{iclr2027_conference}

\clearpage
\appendix

\section{Proofs of Main Results}
\label{sec:proofs}
\setcounter{theorem}{0}
\renewcommand{\thetheorem}{\thesection.\arabic{theorem}}
\setcounter{lemma}{0}
\renewcommand{\thelemma}{\thesection.\arabic{lemma}}
\renewcommand{\theHtheorem}{\thetheorem}
\renewcommand{\theHlemma}{\thelemma}

We collect here the proofs of the results stated in the main text. Throughout, we fix an anchor state $\boldsymbol{s}$ and write $N_s = |\mathcal{G}(\boldsymbol{s})| \geqslant 2$ for the size of its step-level group. We assume the action--return samples are i.i.d.\ conditional on $\boldsymbol{s}$,
\[
(\boldsymbol{a}^{(j)}, R^{(j)}) \stackrel{\text{i.i.d.}}{\sim} \pi(\boldsymbol{a}\mid\boldsymbol{s})\,p(R\mid\boldsymbol{s},\boldsymbol{a}),
\]
and that $\operatorname{Var}[R\mid\boldsymbol{s}] < \infty$. The state and action values are
\[
V(\boldsymbol{s}) = \mathbb{E}_{\pi}[R\mid\boldsymbol{s}], \qquad Q(\boldsymbol{s},\boldsymbol{a}) = \mathbb{E}_{\pi}[R\mid\boldsymbol{s},\boldsymbol{a}], \qquad A_S^{*}(\boldsymbol{s},\boldsymbol{a}) = Q(\boldsymbol{s},\boldsymbol{a}) - V(\boldsymbol{s}),
\]
and the return variance decomposes by the law of total variance as
\[
\operatorname{Var}(R\mid\boldsymbol{s}) = \operatorname{Var}_{\boldsymbol{a}}[Q(\boldsymbol{s},\boldsymbol{a})\mid\boldsymbol{s}] + \mathbb{E}_{\boldsymbol{a}}[\operatorname{Var}(R\mid\boldsymbol{s},\boldsymbol{a})\mid\boldsymbol{s}].
\]
Two identities follow immediately from $V(\boldsymbol{s}) = \mathbb{E}_{\boldsymbol{a}}[Q(\boldsymbol{s},\boldsymbol{a})\mid\boldsymbol{s}]$: $\mathbb{E}_{\boldsymbol{a}}[A_S^{*}(\boldsymbol{s},\boldsymbol{a})\mid\boldsymbol{s}] = 0$ and $\mathbb{E}_{\boldsymbol{a}}[A_S^{*}(\boldsymbol{s},\boldsymbol{a})^{2}\mid\boldsymbol{s}] = \operatorname{Var}_{\boldsymbol{a}}[Q(\boldsymbol{s},\boldsymbol{a})\mid\boldsymbol{s}]$. The step-level estimate and its target are
\[
A_S(\boldsymbol{s},\boldsymbol{a}^{(i)}) = R^{(i)} - \hat{V}_{N_s}, \qquad \hat{V}_{N_s} = \frac{1}{N_s}\sum_{j=1}^{N_s} R^{(j)}, \qquad A_S^{*}(\boldsymbol{s},\boldsymbol{a}^{(i)}) = Q(\boldsymbol{s},\boldsymbol{a}^{(i)}) - V(\boldsymbol{s}),
\]
where the group mean $\hat{V}_{N_s}$ includes sample $i$ itself.

\begin{theorem}[Theorem~\ref{thm:w-star} in the main text]
\label{thm:appendix-w-star}
Fix an anchor state $\boldsymbol{s}$ and a step-level group $\mathcal{G}(\boldsymbol{s})$ of size $N_s \geqslant 2$ with the noise model in Eq.~\eqref{eq:noisy-estimate}, with a strictly positive second moment of the estimate, $\mathbb{E}[A_S(\boldsymbol{s},\boldsymbol{a})^{2}\mid\boldsymbol{s}] > 0$. Then the minimizer of Eq.~\eqref{eq:ms-objective} is
\begin{equation}
 w^{*}(\boldsymbol{s}) = \frac{\mathbb{E}\!\left[A_S(\boldsymbol{s},\boldsymbol{a})\,A_S^{*}(\boldsymbol{s},\boldsymbol{a})\,\middle|\,\boldsymbol{s}\right]}{\mathbb{E}\!\left[A_S(\boldsymbol{s},\boldsymbol{a})^{2}\,\middle|\,\boldsymbol{s}\right]}.
\end{equation}
\end{theorem}

\begin{proof}
We minimize $\mathcal{L}(w) = \mathbb{E}\!\left[(w\,A_S - A_S^{*})^{2}\mid\boldsymbol{s}\right]$ over $w\in\mathbb{R}$, the expectation taken over $i$ and $\boldsymbol{a}^{(i)}\sim\pi(\cdot\mid\boldsymbol{s})$. Expanding the square,
\[
\mathcal{L}(w) = w^{2}\,\mathbb{E}[A_S^{2}\mid\boldsymbol{s}] - 2w\,\mathbb{E}[A_S\,A_S^{*}\mid\boldsymbol{s}] + \mathbb{E}[A_S^{*2}\mid\boldsymbol{s}].
\]
This is a convex quadratic in $w$ (its curvature $2\,\mathbb{E}[A_S^{2}\mid\boldsymbol{s}]$ is strictly positive by assumption), so its minimizer is the stationary point. Differentiating and setting to zero,
\[
\frac{\partial \mathcal{L}}{\partial w} = 2w\,\mathbb{E}[A_S^{2}\mid\boldsymbol{s}] - 2\,\mathbb{E}[A_S\,A_S^{*}\mid\boldsymbol{s}] = 0
\quad\Longrightarrow\quad
w^{*}(\boldsymbol{s}) = \frac{\mathbb{E}[A_S\,A_S^{*}\mid\boldsymbol{s}]}{\mathbb{E}[A_S^{2}\mid\boldsymbol{s}]}.
\]
This is the orthogonal projection of the truth $A_S^{*}$ onto the one-dimensional span of the estimate $A_S$; it requires only $\mathbb{E}[A_S^{2}\mid\boldsymbol{s}]>0$, and no unbiasedness assumption is needed for this projection to hold. The role of unbiasedness, established next, is to make the closed form of the numerator computable.
\end{proof}

\begin{lemma}
\label{lem:eta-decomp}
Fix an anchor state $\boldsymbol{s}$ and a step-level group $\mathcal{G}(\boldsymbol{s})$ of size $N_s\geqslant 2$, with action--return samples $(\boldsymbol{a}^{(j)},R^{(j)})$ i.i.d.\ conditional on $\boldsymbol{s}$, and define $\hat{V}_{N_s}=\tfrac{1}{N_s}\sum_{j=1}^{N_s}R^{(j)}$ (the group mean, which includes sample $i$ itself), $A_S(\boldsymbol{s},\boldsymbol{a}^{(i)})=R^{(i)}-\hat{V}_{N_s}$, $A_S^{*}(\boldsymbol{s},\boldsymbol{a}^{(i)})=Q(\boldsymbol{s},\boldsymbol{a}^{(i)})-V(\boldsymbol{s})$, and $\eta_i=A_S(\boldsymbol{s},\boldsymbol{a}^{(i)})-A_S^{*}(\boldsymbol{s},\boldsymbol{a}^{(i)})$. Then $\eta_i$ admits the three-term decomposition
\begin{equation}
\eta_i = \frac{N_s-1}{N_s}\bigl(R^{(i)} - Q(\boldsymbol{s},\boldsymbol{a}^{(i)})\bigr) \;-\; \frac{A_S^{*}(\boldsymbol{s},\boldsymbol{a}^{(i)})}{N_s} \;-\; \frac{1}{N_s}\sum_{j\neq i}\bigl(R^{(j)} - V(\boldsymbol{s})\bigr),
\end{equation}
i.e.\ Eq.~\eqref{eq:eta-decomp} of the main text, where term~left is the sampling noise of the current rollout, term~middle is a $1/N_s$ shrinkage of the true step advantage, and term~right is the average deviation of the other group members from the state value. Its conditional and marginal means are
\begin{align}
\mathbb{E}_{i}[\eta_i\mid\boldsymbol{s},\boldsymbol{a}^{(i)}] &= -\frac{A_S^{*}(\boldsymbol{s},\boldsymbol{a}^{(i)})}{N_s}, \label{eq:cond-bias} \\
\mathbb{E}_{i,\boldsymbol{a}\sim\pi(\cdot\mid\boldsymbol{s})}[\eta_i\mid\boldsymbol{s}] &= 0. \label{eq:unbiased}
\end{align}
\end{lemma}

\begin{proof}
From $A_S(\boldsymbol{s},\boldsymbol{a}^{(i)}) = R^{(i)} - \hat{V}_{N_s}$ and $A_S^{*}(\boldsymbol{s},\boldsymbol{a}^{(i)}) = Q(\boldsymbol{s},\boldsymbol{a}^{(i)}) - V(\boldsymbol{s})$,
\[
\eta_i = A_S(\boldsymbol{s},\boldsymbol{a}^{(i)}) - A_S^{*}(\boldsymbol{s},\boldsymbol{a}^{(i)}) = \bigl(R^{(i)} - Q(\boldsymbol{s},\boldsymbol{a}^{(i)})\bigr) - \bigl(\hat{V}_{N_s} - V(\boldsymbol{s})\bigr).
\]
Split the group-mean residual into the $i$-th and the remaining terms,
\[
\hat{V}_{N_s} - V(\boldsymbol{s}) = \frac{1}{N_s}\sum_{j=1}^{N_s}(R^{(j)} - V(\boldsymbol{s})) = \frac{1}{N_s}(R^{(i)} - V(\boldsymbol{s})) + \frac{1}{N_s}\sum_{j\neq i}(R^{(j)} - V(\boldsymbol{s})),
\]
so that
\[
\eta_i = \bigl(R^{(i)} - Q(\boldsymbol{s},\boldsymbol{a}^{(i)})\bigr) - \frac{1}{N_s}(R^{(i)} - V(\boldsymbol{s})) - \frac{1}{N_s}\sum_{j\neq i}(R^{(j)} - V(\boldsymbol{s})).
\]
Grouping the first two terms,
\begin{equation*}
\begin{aligned}
R^{(i)} - Q(\boldsymbol{s},\boldsymbol{a}^{(i)}) - \frac{1}{N_s}(R^{(i)} - V(\boldsymbol{s}))
= \frac{N_s-1}{N_s}\bigl(R^{(i)} - Q(\boldsymbol{s},\boldsymbol{a}^{(i)})\bigr) + \frac{1}{N_s}\bigl(V(\boldsymbol{s}) - Q(\boldsymbol{s},\boldsymbol{a}^{(i)})\bigr) \\
= \frac{N_s-1}{N_s}\bigl(R^{(i)} - Q(\boldsymbol{s},\boldsymbol{a}^{(i)})\bigr) - \frac{1}{N_s}A_S^{*}(\boldsymbol{s},\boldsymbol{a}^{(i)}),
\end{aligned}
\end{equation*}
which yields the three-term decomposition~\eqref{eq:eta-decomp}:
\[
\eta_i = \frac{N_s-1}{N_s}\bigl(R^{(i)} - Q(\boldsymbol{s},\boldsymbol{a}^{(i)})\bigr) - \frac{A_S^{*}(\boldsymbol{s},\boldsymbol{a}^{(i)})}{N_s} - \frac{1}{N_s}\sum_{j\neq i}\bigl(R^{(j)} - V(\boldsymbol{s})\bigr).
\]

We now take conditional expectations. Given $(\boldsymbol{s},\boldsymbol{a}^{(i)})$:
\begin{itemize}
\item Term~left has mean $\frac{N_s-1}{N_s}\bigl(\mathbb{E}[R^{(i)}\mid\boldsymbol{s},\boldsymbol{a}^{(i)}] - Q(\boldsymbol{s},\boldsymbol{a}^{(i)})\bigr) = \frac{N_s-1}{N_s}\bigl(Q(\boldsymbol{s},\boldsymbol{a}^{(i)}) - Q(\boldsymbol{s},\boldsymbol{a}^{(i)})\bigr) = 0$.
\item For term~right, each $R^{(j)}$ with $j\neq i$ is independent of $\boldsymbol{a}^{(i)}$ and has conditional mean $\mathbb{E}[R^{(j)}\mid\boldsymbol{s},\boldsymbol{a}^{(i)}] = \mathbb{E}[R^{(j)}\mid\boldsymbol{s}] = V(\boldsymbol{s})$, so $\mathbb{E}[R^{(j)} - V(\boldsymbol{s})\mid\boldsymbol{s},\boldsymbol{a}^{(i)}] = 0$, and term~(iii) vanishes.
\end{itemize}
Only term~middle survives, giving the conditional bias~\eqref{eq:cond-bias}:
\[
\mathbb{E}_{i}[\eta_i\mid\boldsymbol{s},\boldsymbol{a}^{(i)}] = -\frac{A_S^{*}(\boldsymbol{s},\boldsymbol{a}^{(i)})}{N_s}.
\]
Marginalizing once more over $\boldsymbol{a}^{(i)}\sim\pi(\cdot\mid\boldsymbol{s})$ and using $\mathbb{E}_{\boldsymbol{a}}[A_S^{*}(\boldsymbol{s},\boldsymbol{a})\mid\boldsymbol{s}] = 0$,
\[
\mathbb{E}_{i,\boldsymbol{a}}[\eta_i\mid\boldsymbol{s}] = -\frac{1}{N_s}\,\mathbb{E}_{\boldsymbol{a}}[A_S^{*}(\boldsymbol{s},\boldsymbol{a})\mid\boldsymbol{s}] = 0,
\]
which is the state-level unbiasedness~\eqref{eq:unbiased}. In words: the estimate is biased by a $1/N_s$ factor at a fixed action, but unbiased once averaged over the action distribution at $\boldsymbol{s}$.
\end{proof}

\begin{lemma}[Lemma~\ref{lem:closed-form} in the main text]
\label{lem:appendix-closed-form}
Assume the action--return pairs $(\boldsymbol{a}^{(j)},R^{(j)})$ are i.i.d.\ conditional on $\boldsymbol{s}$ with finite return variance, in addition to the assumptions of Theorem~\ref{thm:w-star}. For the step-level estimate and its target defined in~\eqref{eq:step-adv} and~\eqref{eq:true-adv}, the numerator and denominator of~\eqref{eq:w-star} satisfy
\begin{align*}
\mathbb{E}[A_S\,A_S^{*}\mid\boldsymbol{s}] &= \frac{N_s-1}{N_s}\,\operatorname{Var}_{\boldsymbol{a}\sim\pi(\cdot\mid\boldsymbol{s})}\!\bigl[Q(\boldsymbol{s},\boldsymbol{a})\bigr], \\
\mathbb{E}[A_S^{2}\mid\boldsymbol{s}] &= \frac{N_s-1}{N_s}\,\operatorname{Var}(R\mid\boldsymbol{s}).
\end{align*}
\end{lemma}

\begin{proof}
Fix any sample index $i$ and write $A_S = A_S(\boldsymbol{s},\boldsymbol{a}^{(i)})$ and $A_S^{*} = A_S^{*}(\boldsymbol{s},\boldsymbol{a}^{(i)})$. All expectations below are over the sampled group, conditional on the indicated state or state--action pair. Since the group mean includes sample $i$ itself,
\[
A_S = R^{(i)} - \hat{V}_{N_s}
= \frac{N_s-1}{N_s}R^{(i)} - \frac{1}{N_s}\sum_{j\neq i}R^{(j)}.
\]
The two terms on the right are independent conditional on $\boldsymbol{s}$.

\emph{Denominator.} Every return has conditional mean $V(\boldsymbol{s})$, so
\[
\mathbb{E}[A_S\mid\boldsymbol{s}]
= \frac{N_s-1}{N_s}V(\boldsymbol{s}) - \frac{N_s-1}{N_s}V(\boldsymbol{s}) = 0.
\]
Thus the second moment equals the variance. Conditional independence and the identical return distributions give
\begin{align*}
\mathbb{E}[A_S^{2}\mid\boldsymbol{s}]
&= \operatorname{Var}(A_S\mid\boldsymbol{s}) \\
&= \Bigl(\frac{N_s-1}{N_s}\Bigr)^{2}\operatorname{Var}(R\mid\boldsymbol{s})
+ \frac{N_s-1}{N_s^{2}}\operatorname{Var}(R\mid\boldsymbol{s}) \\
&= \frac{N_s-1}{N_s}\,\operatorname{Var}(R\mid\boldsymbol{s}),
\end{align*}
which proves denominator.

\emph{Numerator.} Conditional on $(\boldsymbol{s},\boldsymbol{a}^{(i)})$, the current return has mean $Q(\boldsymbol{s},\boldsymbol{a}^{(i)})$, whereas each other return still has mean $V(\boldsymbol{s})$ by independence. Therefore,
\begin{align*}
\mathbb{E}[A_S\mid\boldsymbol{s},\boldsymbol{a}^{(i)}]
&= \frac{N_s-1}{N_s}Q(\boldsymbol{s},\boldsymbol{a}^{(i)}) - \frac{N_s-1}{N_s}V(\boldsymbol{s}) \\
&= \frac{N_s-1}{N_s}A_S^{*}.
\end{align*}
Since $A_S^{*}$ is determined by $(\boldsymbol{s},\boldsymbol{a}^{(i)})$, the law of iterated expectations yields
\begin{align*}
\mathbb{E}[A_S\,A_S^{*}\mid\boldsymbol{s}]
&= \mathbb{E}\!\left[A_S^{*}\,\mathbb{E}[A_S\mid\boldsymbol{s},\boldsymbol{a}^{(i)}]\,\middle|\,\boldsymbol{s}\right] \\
&= \frac{N_s-1}{N_s}\,\mathbb{E}[(A_S^{*})^{2}\mid\boldsymbol{s}] \\
&= \frac{N_s-1}{N_s}\,\operatorname{Var}_{\boldsymbol{a}\sim\pi(\cdot\mid\boldsymbol{s})}\!\bigl[Q(\boldsymbol{s},\boldsymbol{a})\bigr].
\end{align*}
The last equality follows from $A_S^{*}=Q(\boldsymbol{s},\boldsymbol{a}^{(i)})-V(\boldsymbol{s})$ and $V(\boldsymbol{s})=\mathbb{E}_{\boldsymbol{a}\sim\pi(\cdot\mid\boldsymbol{s})}[Q(\boldsymbol{s},\boldsymbol{a})]$. This proves and completes the proof of Lemma~\ref{lem:closed-form} in the main text.
\end{proof}

\begin{lemma}[Conditional return-variance decomposition]
\label{lem:return-variance}
Fix a state $\boldsymbol{s}$ and a policy $\pi$, and assume $\mathbb{E}[R^{2}\mid\boldsymbol{s}]<\infty$. Let $Q(\boldsymbol{s},\boldsymbol{a})=\mathbb{E}[R\mid\boldsymbol{s},\boldsymbol{a}]$ and $V(\boldsymbol{s})=\mathbb{E}[R\mid\boldsymbol{s}]$. Then the total return variance at $\boldsymbol{s}$ satisfies~\eqref{eq:tva}:
\[
\operatorname{Var}(R\mid\boldsymbol{s})
= \operatorname{Var}_{\boldsymbol{a}\sim\pi(\cdot\mid\boldsymbol{s})}\!\bigl[Q(\boldsymbol{s},\boldsymbol{a})\bigr]
+ \mathbb{E}_{\boldsymbol{a}\sim\pi(\cdot\mid\boldsymbol{s})}\!\bigl[\operatorname{Var}(R\mid\boldsymbol{s},\boldsymbol{a})\bigr].
\]
\end{lemma}

\begin{proof}
By the law of iterated expectations,
\[
V(\boldsymbol{s})
= \mathbb{E}_{\boldsymbol{a}\sim\pi(\cdot\mid\boldsymbol{s})}\!\bigl[Q(\boldsymbol{s},\boldsymbol{a})\bigr].
\]
Decompose the centered return as
\[
R-V(\boldsymbol{s})
= \bigl(Q(\boldsymbol{s},\boldsymbol{a})-V(\boldsymbol{s})\bigr)
+ \bigl(R-Q(\boldsymbol{s},\boldsymbol{a})\bigr).
\]
The first term depends only on the current action when $\boldsymbol{s}$ is fixed, whereas the second is the difference between the realized return $R$ and its conditional expected return $Q(\boldsymbol{s},\boldsymbol{a})$. By the definition of $Q$,
\[
\mathbb{E}\!\bigl[R-Q(\boldsymbol{s},\boldsymbol{a})\mid\boldsymbol{s},\boldsymbol{a}\bigr]=0.
\]
Since $Q(\boldsymbol{s},\boldsymbol{a})-V(\boldsymbol{s})$ is fixed conditional on $(\boldsymbol{s},\boldsymbol{a})$, conditioning on the action and then averaging shows that the cross term vanishes:
\[
\mathbb{E}\!\left[\bigl(Q(\boldsymbol{s},\boldsymbol{a})-V(\boldsymbol{s})\bigr)\bigl(R-Q(\boldsymbol{s},\boldsymbol{a})\bigr)\,\middle|\,\boldsymbol{s}\right]=0.
\]
Squaring the centered-return decomposition and taking conditional expectations therefore gives
\begin{align*}
\operatorname{Var}(R\mid\boldsymbol{s})
&= \mathbb{E}\!\left[\bigl(R-V(\boldsymbol{s})\bigr)^{2}\,\middle|\,\boldsymbol{s}\right] \\
&= \mathbb{E}_{\boldsymbol{a}\sim\pi(\cdot\mid\boldsymbol{s})}\!\left[\bigl(Q(\boldsymbol{s},\boldsymbol{a})-V(\boldsymbol{s})\bigr)^{2}\right]
+ \mathbb{E}\!\left[\bigl(R-Q(\boldsymbol{s},\boldsymbol{a})\bigr)^{2}\,\middle|\,\boldsymbol{s}\right] \\
&= \operatorname{Var}_{\boldsymbol{a}\sim\pi(\cdot\mid\boldsymbol{s})}\!\bigl[Q(\boldsymbol{s},\boldsymbol{a})\bigr]
+ \mathbb{E}_{\boldsymbol{a}\sim\pi(\cdot\mid\boldsymbol{s})}\!\bigl[\operatorname{Var}(R\mid\boldsymbol{s},\boldsymbol{a})\bigr].
\end{align*}
In the last equality, the first term is the variance of $Q(\boldsymbol{s},\boldsymbol{a})$ across actions because its mean is $V(\boldsymbol{s})$. For the second term, $\bigl(R-Q(\boldsymbol{s},\boldsymbol{a})\bigr)^{2}$ is the squared difference between the realized return and its expected value for a fixed state--action pair. Its expectation conditional on $(\boldsymbol{s},\boldsymbol{a})$ is, by definition, $\operatorname{Var}(R\mid\boldsymbol{s},\boldsymbol{a})$; averaging over actions gives the expected conditional return variance in~\eqref{eq:tva}. This completes the proof.
\end{proof}

\begin{theorem}[Theorem~\ref{thm:w-computable} in the main text]
\label{thm:appendix-w-computable}
Under the assumptions of Theorem~\ref{thm:w-star} and Lemma~\ref{lem:closed-form}, the optimal per-state shrinkage weight~\eqref{eq:w-star} is
\[
w^{*}(\boldsymbol{s}) = \frac{\operatorname{Var}_{\boldsymbol{a}\sim\pi(\cdot\mid\boldsymbol{s})}\!\bigl[Q(\boldsymbol{s},\boldsymbol{a})\bigr]}{\operatorname{Var}_{\boldsymbol{a}\sim\pi(\cdot\mid\boldsymbol{s})}\!\bigl[Q(\boldsymbol{s},\boldsymbol{a})\bigr] + \mathbb{E}_{\boldsymbol{a}\sim\pi(\cdot\mid\boldsymbol{s})}\!\bigl[\operatorname{Var}(R\mid\boldsymbol{s},\boldsymbol{a})\bigr]}\;\in[0,1],
\]
or, equivalently,
\[
w^{*}(\boldsymbol{s}) = 1 - \frac{\mathbb{E}_{\boldsymbol{a}\sim\pi(\cdot\mid\boldsymbol{s})}\!\bigl[\operatorname{Var}(R\mid\boldsymbol{s},\boldsymbol{a})\bigr]}{\operatorname{Var}(R\mid\boldsymbol{s})}.
\]
\end{theorem}

\begin{proof}

\emph{Existence and uniqueness.}
 A nonzero step-level correction requires the group to contain at least two different return values.  Thus, the observed advantages of an equal-return group cannot determine a unique weight.

 The theorem concerns the population mean-square objective, whose expectation averages over sampled groups rather than conditioning on one realized group. For any fixed group member, the estimate $A_S$ has zero conditional mean because every sampled return has mean $V(\boldsymbol{s})$. By conditional independence of the returns,
\[
\mathbb{E}[A_S^{2}\mid\boldsymbol{s}]
= \left[\Bigl(\frac{N_s-1}{N_s}\Bigr)^{2}+\frac{N_s-1}{N_s^{2}}\right]\operatorname{Var}(R\mid\boldsymbol{s})
= \frac{N_s-1}{N_s}\,\operatorname{Var}(R\mid\boldsymbol{s})>0.
\]
The strict inequality is an assumption of Theorem~\ref{thm:w-star}. Since $N_s\geqslant 2$, this identity also implies $\operatorname{Var}(R\mid\boldsymbol{s})>0$. The finite return variance also ensures finite second moments of $A_S$ and $A_S^{*}$. Thus the mean-square objective is a finite, strictly convex quadratic in $w$, and Theorem~\ref{thm:appendix-w-star} gives its unique minimizer as the ratio in~\eqref{eq:w-star}.

An individual sampled group may still have identical returns even when $\operatorname{Var}(R\mid\boldsymbol{s})>0$; this does not contradict population-level uniqueness. By contrast, if $\operatorname{Var}(R\mid\boldsymbol{s})=0$, then $R=V(\boldsymbol{s})$ almost surely and $Q(\boldsymbol{s},\boldsymbol{a})=V(\boldsymbol{s})$ for $\pi(\cdot\mid\boldsymbol{s})$-almost every action. Both $A_S$ and $A_S^{*}$ are then zero almost surely, every weight minimizes the objective, and the variance-ratio formulas are undefined. This degenerate case is excluded by the positive-second-moment assumption.

\emph{Interpretive form~\eqref{eq:w-signal}.} Substituting the two moment identities from Lemma~\ref{lem:appendix-closed-form} into~\eqref{eq:w-star} cancels the common factor $\frac{N_s-1}{N_s}$:
\[
w^{*}(\boldsymbol{s}) = \frac{\operatorname{Var}_{\boldsymbol{a}}[Q(\boldsymbol{s},\boldsymbol{a})\mid\boldsymbol{s}]}{\operatorname{Var}(R\mid\boldsymbol{s})}.
\]
Expanding the denominator by Lemma~\ref{lem:return-variance} gives~\eqref{eq:w-signal}:
\[
w^{*}(\boldsymbol{s}) = \frac{\operatorname{Var}_{\boldsymbol{a}\sim\pi(\cdot\mid\boldsymbol{s})}[Q(\boldsymbol{s},\boldsymbol{a})]}{\operatorname{Var}_{\boldsymbol{a}\sim\pi(\cdot\mid\boldsymbol{s})}[Q(\boldsymbol{s},\boldsymbol{a})] + \mathbb{E}_{\boldsymbol{a}\sim\pi(\cdot\mid\boldsymbol{s})}[\operatorname{Var}(R\mid\boldsymbol{s},\boldsymbol{a})]}.
\]
The denominator equals the strictly positive total return variance, and both variance components are non-negative. Therefore $w^{*}(\boldsymbol{s})\in[0,1]$. It equals $1$ exactly when the expected conditional return variance is zero, and $0$ exactly when $Q(\boldsymbol{s},\boldsymbol{a})$ is constant for $\pi(\cdot\mid\boldsymbol{s})$-almost every action.

\emph{Computable form~\eqref{eq:w-computable}.} By the law of total variance~\eqref{eq:tva}, the action-value variance equals the total return variance minus the expected conditional return variance. Substituting this identity yields
\begin{align*}
w^{*}(\boldsymbol{s}) &= \frac{\operatorname{Var}(R\mid\boldsymbol{s}) - \mathbb{E}_{\boldsymbol{a}\sim\pi(\cdot\mid\boldsymbol{s})}[\operatorname{Var}(R\mid\boldsymbol{s},\boldsymbol{a})]}{\operatorname{Var}(R\mid\boldsymbol{s})} \\
&= 1 - \frac{\mathbb{E}_{\boldsymbol{a}\sim\pi(\cdot\mid\boldsymbol{s})}[\operatorname{Var}(R\mid\boldsymbol{s},\boldsymbol{a})]}{\operatorname{Var}(R\mid\boldsymbol{s})}.
\end{align*}
This proves the equivalent computable form~\eqref{eq:w-computable} and completes the proof of Theorem~\ref{thm:w-computable} in the main text.
\end{proof}

\clearpage
\section{Pseudo Code}
\label{sec:adastep-pseudocode}

Algorithm~\ref{alg:adastep} summarizes the AdaStep training procedure, following the rollout, grouping, and policy-update structure of GiGPO~\citep{feng2025gigpo}. AdaStep replaces the fixed coefficient $\omega$ with a coefficient estimated separately for each anchor-state group; the corresponding operations are italicized in the algorithm. Within each task, we use anchor states as hash-table keys to collect step-level groups, and further partition each group by the observed action to estimate within-action return variability. These statistics can be accumulated using counts, sums, and sums of squared returns, without an additional critic or additional rollouts.

\paragraph{Computing the adaptive coefficient.}
For a trajectory $\tau_i$ of length $T_i$, let $R_{i,t}=\sum_{k=t}^{T_i}\gamma^{k-t}r_{i,k}$ denote the discounted return from step $t$. Write $\tilde{s}_{i,t}$ for its anchor-state key, while $s_{i,t}$ denotes the state or context conditioning the policy. For each task, define the index group $\mathcal{G}(\tilde{s})=\{(i,t):\tilde{s}_{i,t}=\tilde{s}\}$. Let $\mathcal{A}_{\tilde{s}}$ be the set of distinct actions observed in the group, and define
\begin{equation}
\begin{aligned}
\mathcal{G}(\tilde{s},a)
  &=\{(i,t)\in\mathcal{G}(\tilde{s}):a_{i,t}=a\},\\
N_{\tilde{s}}&=|\mathcal{G}(\tilde{s})|,\qquad
n_{\tilde{s},a}=|\mathcal{G}(\tilde{s},a)|,\qquad
\widehat p(a\mid\tilde{s})=\frac{n_{\tilde{s},a}}{N_{\tilde{s}}}.
\end{aligned}
\label{eq:adastep-action-groups}
\end{equation}
Denote the empirical variance of returns in $\mathcal{G}(\tilde{s})$ by $\widehat v_{\tilde{s}}$, and the empirical variance within $\mathcal{G}(\tilde{s},a)$ by $\widehat v_{\tilde{s},a}$. The variance denominators are $N_{\tilde{s}}$ for the full group and $n_{\tilde{s},a}$ for each action subgroup. Substituting these statistics into Eq.~\eqref{eq:w-computable} gives the empirical coefficient
\begin{equation}
\widehat w(\tilde{s})
=1-\frac{\displaystyle\sum_{a\in\mathcal{A}_{\tilde{s}}}
\widehat p(a\mid\tilde{s})\,\widehat v_{\tilde{s},a}}
{\widehat v_{\tilde{s}}}.
\label{eq:adastep-empirical-weight}
\end{equation}
Here $\widehat w$ is a finite-sample estimate of the population coefficient $w^*$, and all variance statistics use the unnormalized discounted returns. We evaluate Eq.~\eqref{eq:adastep-empirical-weight} only if $|\mathcal{A}_{\tilde{s}}|\geq 2$, $n_{\tilde{s},a}\geq 2$ for every observed action, and $\widehat v_{\tilde{s}}>0$. Otherwise, we set $\widehat w(\tilde{s})=1$, following the fallback choice in Section~\ref{sec:fallback}. If the group returns are identical, we set the step-level advantages to zero and skip the variance ratio. The combined advantage is
\begin{equation}
A_{i,t}=A_E(\tau_i)+\widehat w(\tilde{s}_{i,t})A_S(s_{i,t},a_{i,t}),
\label{eq:adastep-combined-advantage}
\end{equation}
where $A_S$ uses the benchmark-specific normalization in Appendix~\ref{sec:experimental-details}. All coefficients and advantages are computed from the old-policy rollouts and held fixed during policy optimization.

\clearpage
\begin{algorithm}[H]
\caption{Training LLM Agents with AdaStep}
\label{alg:adastep}
\small
\begin{algorithmic}[1]
\Require Initial policy $\pi_\theta$, fixed reference policy $\pi_{\mathrm{ref}}$, task distribution $p(X)$
\Require Group size $N$, maximum episode length $T_{\max}$, discount $\gamma$, clipping parameter $\epsilon$, KL coefficient $\beta$, and benchmark-specific step-normalization setting
\For{each training iteration}
  \State $\theta_{\mathrm{old}}\gets\theta$; sample a task batch from $p(X)$; initialize buffer $\mathcal B\gets\varnothing$
  \For{each sampled task $x$}
    \State \textbf{Multi-step rollout:} initialize $N$ identical environments for $x$
    \For{$t=1,\ldots,T_{\max}$, until all environments terminate}
      \State Sample $a_{i,t}\sim\pi_{\theta_{\mathrm{old}}}(\cdot\mid s_{i,t},x)$ for each active environment $i$
      \State Execute actions, record $(s_{i,t},a_{i,t},r_{i,t},s_{i,t+1})$, and stop terminated episodes
    \EndFor
    \State Compute $R_E(\tau_i)=\sum_{t=1}^{T_i}r_{i,t}$ and $R_{i,t}=\sum_{k=t}^{T_i}\gamma^{k-t}r_{i,k}$
    \State Compute $A_E(\tau_i)$ using Eq.~\eqref{eq:grpo-adv}; set $A_E=0$ if all episode returns agree
    \State \textbf{Grouping:} hash transitions by anchor state to build $\mathcal G(\tilde{s})$ within task $x$
    \For{each anchor-state group $\mathcal G(\tilde{s})$}
      \State Compute the group return mean $\overline R_{\tilde{s}}$
      \State \textit{Initialize the fallback coefficient $\widehat w(\tilde{s})\gets 1$}
      \If{all returns in $\mathcal G(\tilde{s})$ are identical}
        \State Set $A_S(s_{i,t},a_{i,t})\gets 0$ for all $(i,t)\in\mathcal G(\tilde{s})$
      \Else
        \State Compute $\widehat v_{\tilde{s}}$ and set $A_S(s_{i,t},a_{i,t})\gets R_{i,t}-\overline R_{\tilde{s}}$
        \State Divide $A_S$ by $\sqrt{\widehat v_{\tilde{s}}}$ if the benchmark uses w/ std
        \State \textit{Partition by action to obtain $\mathcal G(\tilde{s},a)$ and counts $n_{\tilde{s},a}$}
        \If{$|\mathcal A_{\tilde{s}}|\geq 2$ and $n_{\tilde{s},a}\geq 2$ for every observed action $a$}
          \State \textit{Compute within-action variances $\widehat v_{\tilde{s},a}$ from the raw returns}
          \State \textit{Set $\widehat w(\tilde{s})\gets 1-\displaystyle\frac{\sum_a(n_{\tilde{s},a}/N_{\tilde{s}})\widehat v_{\tilde{s},a}}{\widehat v_{\tilde{s}}}$}
        \EndIf
      \EndIf
      \State \textit{Combine $A_{i,t}\gets A_E(\tau_i)+\widehat w(\tilde{s})A_S(s_{i,t},a_{i,t})$ for each member}
    \EndFor
    \State Store the task's valid transitions and fixed advantages in $\mathcal B$
  \EndFor
  \State \textbf{Policy update:} optimize $\theta$ on $\mathcal B$ by maximizing Eq.~\eqref{eq:adastep-objective}
\EndFor
\Ensure Trained policy $\pi_\theta$
\end{algorithmic}
\end{algorithm}

\subsection{Policy Optimization Objective}
\label{sec:adastep-objective}

AdaStep uses the clipped policy objective of GiGPO with the adaptive combined advantage in Eq.~\eqref{eq:adastep-combined-advantage}. For an action sampled at state $s_{i,t}$ under task $x$, the importance sampling ratio is
\begin{equation}
\rho_{i,t}(\theta)
=\frac{\pi_\theta(a_{i,t}\mid s_{i,t},x)}
{\pi_{\theta_{\mathrm{old}}}(a_{i,t}\mid s_{i,t},x)}.
\label{eq:adastep-importance-ratio}
\end{equation}
Let $M=\sum_{i=1}^N T_i$ be the number of valid environment steps in a task group. We maximize
\begin{equation}
\begin{aligned}
\mathcal J_{\mathrm{AdaStep}}(\theta)
={}&\mathbb E_{\substack{x\sim p(X)\\
\{\tau_i\}_{i=1}^N\sim\pi_{\theta_{\mathrm{old}}}(\cdot\mid x)}}
\Bigg[\frac{1}{M}\sum_{i=1}^N\sum_{t=1}^{T_i}
\Big(\\[-2pt]
&\min\!\left\{\rho_{i,t}(\theta)A_{i,t},\;
\operatorname{clip}\!\left(\rho_{i,t}(\theta),1-\epsilon,1+\epsilon\right)A_{i,t}\right\}\\
&-\beta D_{\mathrm{KL}}\!\left(
\pi_\theta(\cdot\mid s_{i,t},x)\,\middle\|\,
\pi_{\mathrm{ref}}(\cdot\mid s_{i,t},x)\right)
\Big)\Bigg].
\end{aligned}
\label{eq:adastep-objective}
\end{equation}
The parameter $\epsilon$ defines the importance-ratio clipping interval, and $\beta$ controls the KL penalty toward the fixed reference policy $\pi_{\mathrm{ref}}$. The old policy is refreshed before each rollout iteration. For a task batch, we average the per-task terms in Eq.~\eqref{eq:adastep-objective}. Only valid steps contribute; if every trajectory has length $T$, then $M=NT$. Replacing $\widehat w(\tilde{s}_{i,t})$ with a fixed $\omega$ in Eq.~\eqref{eq:adastep-combined-advantage} recovers the GiGPO advantage and its corresponding clipped objective.

\clearpage
\section{Experimental Details}
\label{sec:experimental-details}

\textbf{Benchmarks.}  We evaluate on three long-horizon benchmarks: ALFWorld~\citep{shridhar2021alfworld}, WebShop~\citep{yao2022webshop}, and ScienceWorld~\citep{wang2022scienceworld}. ALFWorld is a text-based embodied environment in which agents complete multi-step household tasks, such as heating, cleaning, and relocating objects, with sparse rewards upon task completion; we set the maximum number of environment steps to 50. WebShop is a simulated shopping environment with 1.18 million products, requiring agents to search, select product options, and purchase items that satisfy natural-language instructions; we set the maximum number of environment steps to 30. ScienceWorld is a text-based scientific reasoning environment in which agents conduct virtual experiments such as measuring melting points and testing electrical conductivity; we set the maximum number of environment steps to 30. In our experiments, ScienceWorld evaluation is restricted to eight simple tasks.

\textbf{Baselines.} We compare against baselines across three categories:
(1) \emph{Frontier models:} DeepSeek-V4.1-Flash~\citep{xu2026deepseek} and
GLM-5.2~\citep{glm5team2026glm5vibecodingagentic}, evaluated with ReAct prompting~\citep{yao2023react} as reference points
for frontier model capabilities.
(2) \emph{Prompting methods:} ReAct~\citep{yao2023react}, which interleaves reasoning and actions,
and Reflexion~\citep{shinn2023reflexion}, which uses verbal self-reflection and
episodic memory to improve subsequent attempts without updating model parameters.
(3) \emph{RL training methods:} PPO~\citep{schulman2017ppo}, a standard
actor--critic algorithm; GRPO~\citep{shao2024deepseekmath} and
GiGPO~\citep{feng2025gigpo}, which estimate relative advantages from trajectory
groups without learned critics, with GiGPO additionally using step-level groups
for fine-grained credit assignment; and HGPO~\citep{he2026hgpo}, which incorporates
historical context into anchor-state grouping and adaptively aggregates advantages
across hierarchical groups to reduce bias in step-level advantage estimation.

\textbf{Implementation.} We use Qwen2.5-7B-Instruct~\citep{qwen2025qwen25technicalreport}, Qwen3-4B, and Qwen3-1.7B~\citep{yang2025qwen3technicalreport} as backbone models. All Qwen3 models are run with their built-in thinking mode disabled throughout training and evaluation. For fair comparison, all RL methods share the same training pipeline and common hyperparameters specified below. We follow the hyperparameter configuration of GiGPO~\citep{feng2025gigpo} for ALFWorld and WebShop, with the WebShop episode limit extended from 15 to 30 steps. AdaStep computes its adaptive weights from the unnormalized return statistics of step-level groups, replacing GiGPO's fixed coefficient $\omega$ without introducing additional tunable hyperparameters.
For HGPO, we set the maximum historical context length to $K=2$ throughout our experiments.

\textbf{Hyperparameters for ALFWorld.} All RL methods use a maximum prompt length of 2,048 tokens and a maximum response length of 512 tokens, with at most 50 environment steps per episode. Training uses a rule-based task reward of $10$ for success and $0$ for failure, together with a reward penalty of $-0.1$ for each invalid action. These practical rewards rescale the theoretical binary success signal and add an invalid-action penalty. For group-based methods (RLOO, GRPO, GiGPO, HGPO, and AdaStep), each rollout batch contains 16 distinct task groups with 8 trajectories per group, yielding $16\times 8=128$ environment instances. PPO instead collects rollouts from 128 separate environment instances without arranging them into trajectory groups. The sampling temperature is $1.0$ during training rollouts and $0.4$ during validation.

The actor learning rate is $1\times 10^{-6}$ for all RL methods. PPO additionally trains a critic with a learning rate of $1\times 10^{-5}$; the group-based methods do not use a critic. We use a mini-batch size of 256 and a KL-divergence loss coefficient of $0.01$. For the GiGPO baseline, the step-level advantage coefficient is fixed at $\omega=1$ without further tuning, and the return discount factor is $\gamma=0.95$. For the main results in Table~\ref{tab:main}, GiGPO, HGPO, and AdaStep compute step-level advantages with standard-deviation normalization (\emph{w/ std}).

\textbf{Hyperparameters for WebShop.} All RL methods use a maximum prompt length of 4,096 tokens and a maximum response length of 512 tokens. Each episode is limited to 30 environment steps. We use the same rule-based reward as in ALFWorld: $10$ for task success, $0$ for failure, and a penalty of $-0.1$ for each invalid action. Each group-based method samples 16 task groups per rollout batch, with 8 trajectories in each group, for a total of 128 environment instances. PPO uses 128 separate environment instances. The rollout and validation temperatures are $1.0$ and $0.4$, respectively.

The actor learning rate is $1\times 10^{-6}$, and the critic learning rate, applicable only to PPO, is $1\times 10^{-5}$. The mini-batch size is 64, and the KL-divergence loss coefficient is $0.01$. As in ALFWorld, the GiGPO baseline uses a fixed step-level advantage coefficient $\omega=1$ without additional tuning and a discount factor $\gamma=0.95$. For the main results in Table~\ref{tab:main}, GiGPO, HGPO, and AdaStep compute step-level advantages without standard-deviation normalization (\emph{w/o std}).

\textbf{Hyperparameters for ScienceWorld.} We evaluate on eight simple ScienceWorld tasks. All RL methods use a maximum prompt length of 2,048 tokens and a maximum response length of 512 tokens, with at most 30 environment steps per episode. Training uses a rule-based task reward of $10$ for success and $0$ for failure, together with a reward penalty of $-0.1$ for each invalid action. For group-based methods (RLOO, GRPO, GiGPO, HGPO, and AdaStep), each rollout batch contains 16 distinct task groups with 8 trajectories per group, yielding $16\times 8=128$ environment instances. PPO instead collects rollouts from 128 separate environment instances without arranging them into trajectory groups. The sampling temperature is $1.0$ during training rollouts and $0.4$ during validation.

The actor learning rate is $1\times 10^{-6}$ for all RL methods. PPO additionally trains a critic with a learning rate of $1\times 10^{-5}$; the group-based methods do not use a critic. We use a mini-batch size of 256 and a KL-divergence loss coefficient of $0.01$. For the GiGPO baseline, the step-level advantage coefficient is fixed at $\omega=1$ without further tuning, and the return discount factor is $\gamma=0.95$. For the main results in Table~\ref{tab:main}, GiGPO, HGPO, and AdaStep compute step-level advantages with standard-deviation normalization (\emph{w/ std}).

\textbf{Computing details.} All experiments were conducted on 8 NVIDIA H100 GPUs for 160 training steps, with each run taking no more than 16 hours. Results reported with means and standard deviations were computed over three independent runs with different random seeds.

\raggedbottom
\section{Prompt Templates}
\label{sec:prompt-template}

Figures~\ref{fig:webshop-prompt}--\ref{fig:scienceworld-prompt} show the
WebShop, ALFWorld, and ScienceWorld prompt templates. Braced placeholders
are filled with task and interaction information at each step. All three
prompts enclose reasoning within \texttt{<thought>} and \texttt{</thought>}
tags, and the chosen action within \texttt{<action>} and \texttt{</action>}
tags. These delimiters specify the response format. All Qwen3 models are
run with their built-in thinking mode disabled, as stated in
Appendix~\ref{sec:experimental-details}.

\begin{figure}[H]
\centering
\begin{tcolorbox}[
  title={Prompt Template for WebShop},
  colback=black!2,
  colframe=black,
  colbacktitle=black,
  coltitle=white,
  fonttitle=\bfseries,
  fontupper=\normalsize,
  boxrule=0.7pt,
  arc=3mm,
  left=3mm,right=3mm,top=2mm,bottom=2mm,
  before skip=0pt,after skip=0pt
]
You are an expert autonomous agent operating in the WebShop e-commerce environment. Your task is to: \promptfield{task\_description}. Prior to this step, you have already taken \promptfield{step\_count} step(s). Below are the most recent \promptfield{history\_length} observations and the corresponding actions you took: \promptfield{action\_history}. You are now at step \promptfield{current\_step} and your current observation is: \promptfield{current\_observation}. Your admissible actions for the current situation are: [\promptfield{available\_actions}].

\smallskip
Now it's your turn to take one action for the current step. You should first reason step-by-step about the current situation, then think carefully which admissible action best advances the shopping goal. This reasoning process MUST be enclosed within \textcolor{promptthought}{\texttt{<thought> </thought>}} tags. Once you've finished your reasoning, you should choose an admissible action for current step and present it within \textcolor{promptaction}{\texttt{<action> </action>}} tags.
\end{tcolorbox}
\caption{The WebShop prompt template with \texttt{<thought>} and \texttt{</thought>} reasoning delimiters.}
\label{fig:webshop-prompt}
\end{figure}

\begin{figure}[H]
\centering
\begin{tcolorbox}[
  title={Prompt Template for ALFWorld},
  colback=black!2,
  colframe=black,
  colbacktitle=black,
  coltitle=white,
  fonttitle=\bfseries,
  fontupper=\normalsize,
  boxrule=0.7pt,
  arc=3mm,
  left=3mm,right=3mm,top=2mm,bottom=2mm,
  before skip=0pt,after skip=0pt
]
You are an expert agent operating in the ALFRED embodied Environment. Your task is to: \promptfield{task\_description}. Prior to this step, you have already taken \promptfield{step\_count} step(s). Below are the most recent \promptfield{history\_length} observations and the corresponding actions you took: \promptfield{action\_history}. You are now at step \promptfield{current\_step} and your current observation is: \promptfield{current\_observation}. Your admissible actions of the current situation are: [\promptfield{admissible\_actions}].

\smallskip
Now it's your turn to take an action. You should first reason step-by-step about the current situation. This reasoning process MUST be enclosed within \textcolor{promptthought}{\texttt{<thought> </thought>}} tags. Once you've finished your reasoning, you should choose an admissible action for current step and present it within \textcolor{promptaction}{\texttt{<action> </action>}} tags.
\end{tcolorbox}
\caption{The ALFWorld prompt template with \texttt{<thought>} and \texttt{</thought>} reasoning delimiters.}
\label{fig:alfworld-prompt}
\end{figure}

\begin{figure}[H]
\centering
\begin{tcolorbox}[
  title={Prompt Template for ScienceWorld},
  colback=black!2,
  colframe=black,
  colbacktitle=black,
  coltitle=white,
  fonttitle=\bfseries,
  fontupper=\normalsize,
  boxrule=0.7pt,
  arc=3mm,
  left=3mm,right=3mm,top=2mm,bottom=2mm,
  before skip=0pt,after skip=0pt
]
You are an expert agent in a text-based interactive science environment (ScienceWorld).

\smallskip
Your task is: \promptfield{task\_description}

\smallskip
Prior to this step, you have already taken \promptfield{step\_count} step(s).
This reasoning process MUST be enclosed within \textcolor{promptthought}{\texttt{<thought> </thought>}} tags.

\smallskip
Once you've finished your reasoning, write the single action you choose,
enclosed within \textcolor{promptaction}{\texttt{<action> </action>}} tags.
\end{tcolorbox}
\caption{The ScienceWorld prompt template with \texttt{<thought>} and \texttt{</thought>} reasoning delimiters.}
\label{fig:scienceworld-prompt}
\end{figure}

\section{Details of the WebShop Step-Credit Analysis}
\label{sec:credit-diagnostic-details}

This appendix describes the offline diagnostics in
Section~\ref{sec:credit-diagnostics}. We compare coefficient rules on the same
historical return samples, rather than collecting a new set of policy-training
runs for this analysis. The comparison contains 3,666 WebShop anchor-state
groups. We write $G_i$ for the recorded discounted step return of sample $i$
and $a_i$ for its action, and suppress the split index below.

\paragraph{Fit and reference subsets.}
Within each anchor-state group $s$, the return observations for each action
are divided into two halves: a fit subset $\mathcal F_{s,a}$ and a disjoint
reference subset $\mathcal H_{s,a}$. This row-level splitting is repeated
32 times per group, and all four coefficient rules use the same splits.
The fit subset is used to estimate group statistics and the adaptive
coefficient. Here, ``fit'' denotes direct empirical estimation; no additional
value model or critic is trained. The reference subset supplies action-level
empirical targets and is not used to choose the adaptive coefficient.

\paragraph{Estimating the coefficient from the fit subset.}
Let $\mathcal F_s=\bigcup_a\mathcal F_{s,a}$,
$n^{\mathrm{fit}}_{s,a}=|\mathcal F_{s,a}|$, and
$N^{\mathrm{fit}}_s=|\mathcal F_s|$. The fit-subset action frequencies and
return means are
\begin{equation}
\begin{aligned}
\widehat p^{\mathrm{fit}}_{s,a}
&=\frac{n^{\mathrm{fit}}_{s,a}}{N^{\mathrm{fit}}_s},
&\widehat\mu^{\mathrm{fit}}_{s,a}
&=\frac{1}{n^{\mathrm{fit}}_{s,a}}
  \sum_{i\in\mathcal F_{s,a}}G_i,\\
\widehat\mu^{\mathrm{fit}}_s
&=\sum_a\widehat p^{\mathrm{fit}}_{s,a}
  \widehat\mu^{\mathrm{fit}}_{s,a},
&X_i&=G_i-\widehat\mu^{\mathrm{fit}}_s,
\quad i\in\mathcal F_s.
\end{aligned}
\label{eq:diagnostic-fit-means}
\end{equation}
The signal $X_i$ is the unnormalized, return-centered step credit. We estimate
within-action return variation and total return variance using only fit rows:
\begin{equation}
\widehat W^{\mathrm{fit}}_s
=\sum_a\widehat p^{\mathrm{fit}}_{s,a}
  \widehat{\operatorname{Var}}_{\mathcal F_{s,a}}(G),
\qquad
\widehat D^{\mathrm{fit}}_s
=\widehat{\operatorname{Var}}_{\mathcal F_s}(G).
\label{eq:diagnostic-fit-variances}
\end{equation}
For this diagnostic, we apply the empirical variance-ratio rule of
Eq.~\eqref{eq:adastep-empirical-weight} to the fit statistics and clip the
resulting coefficient to $[0,1]$:
\begin{equation}
\widehat w_s
=\operatorname{clip}\!\left(
1-\frac{\widehat W^{\mathrm{fit}}_s}{\widehat D^{\mathrm{fit}}_s},0,1\right).
\label{eq:diagnostic-fit-weight}
\end{equation}
We use the fallback coefficient $\widehat w_s=1$ when the fit subset does not
support coefficient estimation, following Section~\ref{sec:fallback}.
Equal-return groups have zero centered step credit. The four rules apply
$c_sX_i$ with $c_s=1$ (GiGPO), $c_s=0.5$ (fixed scaling), $c_s=0$ (step off),
or $c_s=\widehat w_s$ (AdaStep). Thus, fitting consists of computing empirical
means, variances, and a scalar ratio; the coefficient is not optimized against
the reference MSE.

\paragraph{Scaled within-action variation.}
Let $p_{s,a}$ denote the empirical action-frequency weights used in the
diagnostic, $W_s=\sum_a p_{s,a}\widehat{\operatorname{Var}}(G\mid s,a)$
the within-action return variation, and $D_s>0$ the common total-return
variance scale. These diagnostic quantities are held fixed across all
four coefficient rules within each split; the fit-only statistics above
are used to estimate $\widehat w_s$. The first diagnostic is
\begin{equation}
\mathcal V_s(c_s)=\frac{c_s^2W_s}{D_s}.
\label{eq:diagnostic-variation}
\end{equation}
It measures the within-action fluctuation remaining after scaling the local
credit. The factor $c_s^2$ follows from the scaling property of variance.
Consequently, $c_s=0.5$ retains one quarter of the unscaled variation and
$c_s=0$ removes it entirely. This metric measures the stability of the
weighted credit, not a change in the environment's raw return variance;
its reduction alone does not establish improved estimation.

\paragraph{Held-out action-level reference MSE.}
For each action, compute its empirical mean return using reference rows,
then center these means using the diagnostic's action-frequency weights:
\begin{equation}
\overline G^{\mathrm{ref}}_{s,a}
=\frac{1}{|\mathcal H_{s,a}|}\sum_{j\in\mathcal H_{s,a}}G_j,
\qquad
B_{s,a}=\overline G^{\mathrm{ref}}_{s,a}
-\sum_b p_{s,b}\overline G^{\mathrm{ref}}_{s,b}.
\label{eq:diagnostic-reference}
\end{equation}
The second diagnostic compares each fit-row credit with the reference for
its action, using the same positive variance scale for all coefficient rules:
\begin{equation}
\mathcal E_s(c_s)
=\frac{1}{D_sN^{\mathrm{fit}}_s}
\sum_{i\in\mathcal F_s}\left(c_sX_i-B_{s,a_i}\right)^2.
\label{eq:diagnostic-reference-mse}
\end{equation}
Lower values indicate closer agreement with the held-out empirical
action-level signal. Removing the step component can increase this error
when the reference action advantages are nonzero. The reference itself is
estimated from finitely many rows, so this is a proxy error, not an error
against true action values or ground-truth advantages.

\paragraph{Aggregation and interpretation.}
For each metric, we average over the 32 splits within each group and then
equally over the 3,666 groups. Figure~\ref{fig:webshop-credit-diagnostics}(c,d)
reports $100\,\overline{\mathcal V}(c)/\overline{\mathcal V}(1)$ and
$100\,\overline{\mathcal E}(c)/\overline{\mathcal E}(1)$, so each metric
has its own GiGPO baseline of 100. Before this GiGPO-relative rescaling,
the aggregate diagnostic values for GiGPO and AdaStep are 0.1576 and 0.0564 for within-action
variation, and 0.3313 and 0.2180 for reference MSE. These are descriptive
results on historical samples. Rows may share trajectories, and groups may
recur across training updates; the repeated splits are not independent
training seeds. We therefore do not attach independent-seed confidence
intervals or interpret the proxy error as a direct measure of downstream
policy performance.

\section{Limitations and Future Work}
\label{sec:limitations}

AdaStep estimates its adaptive coefficient from the return statistics of the samples that share an anchor state, so it relies on constructing step-level state groups across multiple rollouts.
In the three benchmarks studied here, the environment state is discrete and exactly reproduced, so identical states can be matched directly and the groups required for estimation are readily available.
In more general settings, such as continuous, partially observable, or open-ended environments, the same situation is rarely revisited verbatim, and how to define and construct state groups that remain suitable for the variance-ratio estimate is an open question.
We therefore leave the design of state abstractions or learned grouping criteria that preserve the reliability interpretation of $w^{*}(\boldsymbol{s})$ to future work.

\end{document}